\documentclass[10pt,journal,compsoc]{IEEEtran}

\usepackage{times}

\usepackage{times}
\usepackage{epsfig}

\usepackage{booktabs}
\usepackage{amsthm}
\theoremstyle{plain}
\usepackage[compress]{cite}
\usepackage{graphics}
\usepackage{epstopdf}
\usepackage{epsfig}
\usepackage{lineno,hyperref}
\usepackage{array}
\usepackage{amsmath,amssymb}
\usepackage{mathrsfs}
\usepackage{subcaption}
\usepackage{multirow}
\usepackage{xcolor}
\usepackage{bm}
\usepackage{eqparbox}
\usepackage{algorithm}
\usepackage{algorithmic}
\usepackage{amsfonts}
\usepackage{enumerate}
\usepackage{indentfirst}
\usepackage{url}
\newtheorem{myTheorem}{Theorem}

\newtheorem{myLemma}{Lemma}
\newtheorem{myAssumption}{Assumption}

\newtheorem{myRemark}{Remark}
\usepackage{bbm}

\allowdisplaybreaks[4]

\usepackage{ragged2e}

\begin{document}

\title{Non-Graph Data Clustering via $\mathcal O(n)$ Bipartite Graph Convolution}

\author{Hongyuan Zhang, Jiankun Shi, Rui Zhang, 
and Xuelong Li$^*$, \IEEEmembership{~Fellow,~IEEE} \thanks{$^*$ Corresponding author}

\thanks{This work is supported by The National Natural Science Foundation of China (No. 61871470).}

\thanks{The authors are with the School of Artificial Intelligence, OPtics and ElectroNics (iOPEN), Northwestern Polytechnical University, Xi'an 710072, P.R. China. 
}

\thanks{
    \copyright 2022 IEEE.  Personal use of this material is permitted.  Permission from IEEE must be obtained for all other uses, in any current or future media, including reprinting/republishing this material for advertising or promotional purposes, creating new collective works, for resale or redistribution to servers or lists, or reuse of any copyrighted component of this work in other works.
}

\thanks{E-mail: hyzhang98@gmail.com, henusjk@163.com, ruizhang8633@gmail.com, li@nwpu.edu.cn.}

}

\markboth{IEEE TRANSACTIONS ON PATTERN ANALYSIS AND MACHINE INTELLIGENCE}{Zhang \MakeLowercase{\textit{et al.}}: 
Non-Graph Data Clustering via $\mathcal O(n)$ Bipartite Graph Convolution}

\IEEEtitleabstractindextext{
\justifying  
    \begin{abstract}
        Since the representative capacity of graph-based clustering methods is usually 
        limited by the graph constructed on the original features,
        it is attractive to find whether graph neural networks (\textit{GNNs}), a strong extension of neural networks to graphs, 
        can be applied to augment the capacity of 
        graph-based clustering methods. 
        The core problems mainly come from two aspects. 
        On the one hand, the graph is unavailable in the most general clustering scenes 
        so that how to construct graph on the non-graph data and the quality of 
        graph is usually the most important part. 
        On the other hand, given $n$ samples, 
        the graph-based clustering methods usually consume at least $\mathcal O(n^2)$ 
        time to build graphs and 
        the graph convolution requires nearly $\mathcal O(n^2)$ for a dense graph and 
        $\mathcal O(|\mathcal{E}|)$ for a sparse one with $|\mathcal{E}|$ edges. 
        Accordingly, both graph-based clustering and GNNs suffer from the severe inefficiency problem. 
        To tackle these problems, 
        we propose a novel clustering method, \textit{AnchorGAE}, 
        with the self-supervised estimation of graph and efficient graph convolution. 
        We first show how to convert a non-graph dataset into a graph dataset, 
        by introducing the generative graph model and anchors. 
        A bipartite graph is built via generating anchors and estimating the 
        connectivity distributions of original points and anchors. 
        We then show that the constructed bipartite 
        graph can reduce the computational complexity of graph convolution 
        from $\mathcal O(n^2)$ and $\mathcal O(|\mathcal{E}|)$ to $\mathcal O(n)$. 
        The succeeding steps for clustering can be easily designed as $\mathcal O(n)$ operations. 
        Interestingly, the anchors naturally lead to siamese architecture with 
        the help of the Markov process. 
        Furthermore, the estimated bipartite graph is updated dynamically 
        according to the features extracted by GNN modules, 
        to promote the quality of the graph by exploiting the high-level 
        information by GNNs. 
        However, we theoretically prove that the self-supervised paradigm frequently 
        results in a collapse that often occurs after 2-3 update iterations in experiments, 
        especially when the model is well-trained. 
        A specific strategy is accordingly designed to prevent the collapse. 
        The experiments support the theoretical analysis and 
        show the superiority of AnchorGAE.
    \end{abstract}

\begin{IEEEkeywords}
    Graph Convolution Network, Efficient Clustering, Anchors, Siamese Network, Self-Supervised Learning.
\end{IEEEkeywords}

}

\maketitle

\section{Introduction}
Graph-based clustering methods \cite{SC,RatioCut,NormalizedCut,CAN,Chen1,NCARL} are 
important in the clustering field since 
they can capture the data topology and group data points non-linearly via constructing a graph. 
For instance, spectral clustering \cite{SC} originates from a relaxed graph-cut problem, 
\textit{e.g.}, Ratio Cut \cite{RatioCut}, Normalized Cut \cite{NormalizedCut}, 
Balanced Cut \cite{BalancedCut}, \textit{etc}.
In graph-based methods, 
it is easier to formally study diverse clustering-related tasks with the 
help of graph theory \cite{SpectralGraphTheory}.

Given $n$ data points,
the graph-based clustering methods suffer the severe inefficiency in practice
due to the construction of graphs, which usually needs $\mathcal O(n^2)$ time without 
any acceleration tricks. 
For some methods such as spectral clustering \cite{SC}, 
the eigenvalue decomposition, an $\mathcal O(n^3)$ operation, is required.
To alleviate the phenomenon, plenty of variants of spectral
clustering have been developed \cite{Nystrom,CSC,KASP}. 
An important variant is to introduce some representative points \cite{LargeGraph} 
to accelerate both graph construction and optimization \cite{LSC,SNC}. 
Some works \cite{MultiView-Anchor-1,MultiView-Anchor-2} recently attempt to extend these methods into 
multi-view scenarios. 
However, the graph constructed by original features may \textbf{only contain 
low-level information} 
(\textit{i.e.}, the representative capacity is limited), 
since the construction is usually based on 
the original features without mapping (\textit{e.g.}, neural networks). 
The limited ability to extract the high-level relationship of data 
results in the inaccurate graphs and thus becomes the bottleneck of 
graph-based clustering methods. 

With the rise of deep learning, the utilization of 
deep neural networks (\textit{DNN}) is widely regarded as an effective way 
to exploit deep information of data via extracting latent features \cite{DEC}. 
Some models \cite{JULE,DEPICT,DualAE} based on convolution neural networks (\textit{CNN}) 
have achieved impressive improvement on clustering, but \textbf{CNN can not be 
applied to non-Euclidean data} \cite{EuclideanData}, such as social networks, recommendation systems, 
word vector, data mining data, \textit{etc}. 
Compared with the Euclidean data 
(\textit{e.g.}, images, videos, \textit{etc.}), 
the most difference of non-Euclidean data is that the relationships of 
features are not regular (or even do not exist).
It indicates that the classical convolution operation could not be utilized   
and the existing CNN-based clustering models can not 
be applied to general clustering scenes. 
Although some deep clustering models \cite{IMSAT,EGAE,StructAE,JULE,DEPICT} have been proposed, 
most methods aim to introduce some widely used modules from supervised 
scenarios to clustering and few models focus on how to utilize 
deep models to further promote the capacity graph-based methods.
SpectralNet \cite{SpectralNet} is an important attempt 
to connect neural networks and classical spectral clustering. 
It utilizes neural networks to promote capacity and simulate spectral decomposition. 

\begin{figure*}[t]
    \centering
    \includegraphics[width=0.95\linewidth]{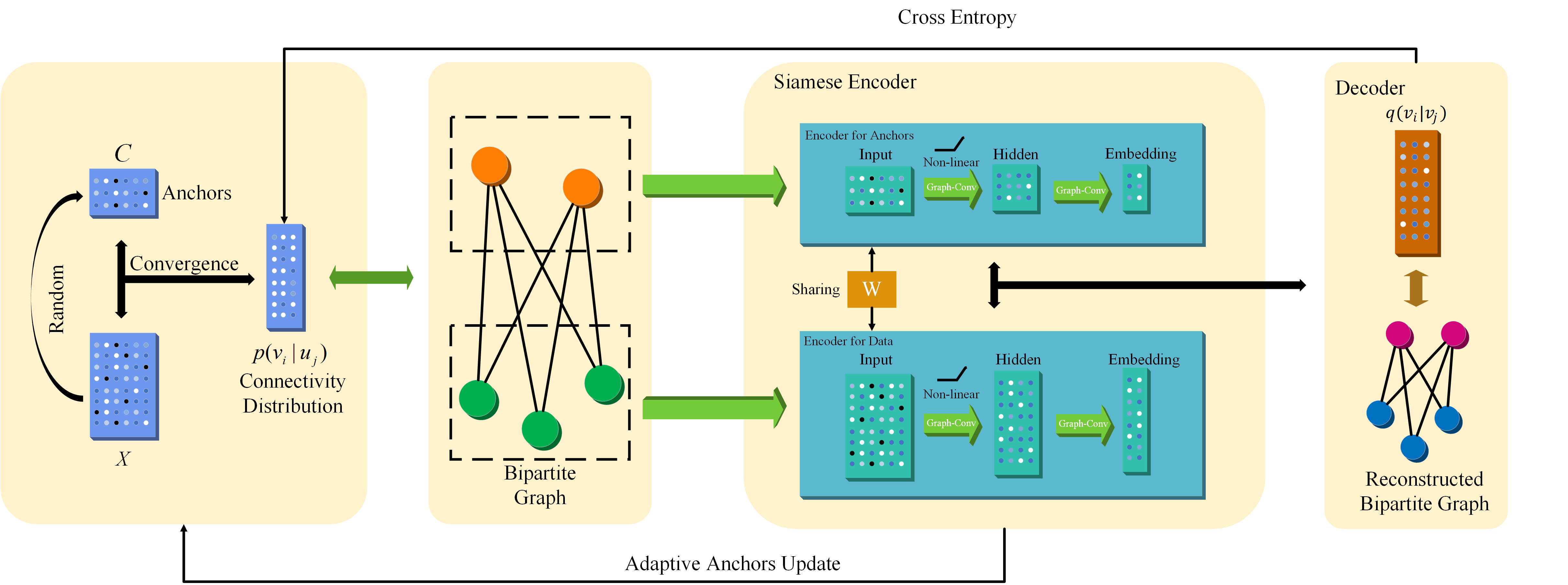}
    \caption{The core idea of AnchorGAE. First, we iteratively calculate the 
                connectivity distribution and anchors. 
                Then they are used as the input of a siamese GNN for acceleration. 
                After training the network, the 
                anchors and connectivity distribution are updated according to the embedding. 
                The two steps are repeated multiple
                times until the high-level information is exploited.}
    \label{figure_framework}
\end{figure*}

Graph neural networks (\textit{GNNs}) \cite{GNN,spectralgcn,GCN,GAT,STGNN-1,GNN-Survey}, 
which attempt to extend the classical neural networks into graph-type data 
with considering the graph structure, have attracted much attention. 
GNN is designed for graph-type data, such as social networks, recommendation 
systems, citation networks, \textit{etc}. 
In these practical scenarios, links are easy and cheap to obtain and 
graphs are provided as prior information. 
For example, in citation networks, each paper is viewed as a node in graph 
and the citation relationships, which are easily available, are modeled as links. 
As GNNs have shown impressive capacity and performance on graph-type data, 
\textbf{it is attractive to employ GNN to promote the capacity of graph-based clustering methods}. 
Nevertheless, there are two major hurdles. 
On the one hand, in the most scenarios (\textit{e.g.}, text, signals, \textit{etc.}), 
\textbf{graphs are not provided as the prior information}, 
which indicates that the existing GNN-based models 
for node clustering cannot be applied on these non-graph data 
without artificial construction of graphs.
It goes back to the same dilemma as the graph-based clustering, 
\textit{i.e.}, how to construct a graph containing enough information. 
On the other hand, \textbf{the inefficiency} is also a well-known problem for GNNs \cite{FastGCN,GraphSAGE}. 
Forward propagation of each node in a GNN layer relies on neighboring samples 
such that the computation complexity of the graph convolution increases 
exponentially with the growth of layers.
For a dense graph, the complexity of a multi-layer GNN is 
almost $\mathcal O(n^2)$ time.
For a sparse one, the computational cost becomes $\mathcal O(|\mathcal{E}|)$ 
where $|\mathcal{E}|$ represents the amount of edges. 
Accordingly, each training step approximately depends on all data points  
and the training is quite time-consuming. 
It also results in the unavailability of batch gradient descent. 
To accelerate GNNs, some works 
focus on how to apply batch gradient descent \cite{ClusterGCN,SGC,StoGCN,FastGCN}. 
The inefficiency problem is also identical to the graph-based clustering. 

The main motivation of this paper is 
\textit{how to \textbf{efficiently} improve the capacity of graph-based clustering via GNNs on \textbf{non-graph} data}. 
The efficiency guarantees the practicability and the non-graph property of data ensures the universality of our models. 
Although SDCN \cite{SDCN} aims to extend GNNs into the data clustering 
task, it simply constructs KNN graphs and thus fails to exploit the high-level information of data. 
The major contributions of the proposed model, namely \textit{\textbf{AnchorGAE}}, are listed as follows
\begin{itemize}
    \item As no graph is provided in general, we aim to construct the graph
            via estimating the connectivity distribution of each sample 
            from a perspective of generative models. 
            Under this generative perspective, anchors are introduced to accelerate 
            the graph construction and computation of GNNs.
            The produced anchors and connectivity distributions lead to 
            a bipartite graph convolution operation 
            so that the computational complexity of each forward 
            propagation of GNNs is reduced from 
            $\mathcal O(n^2)$ or $\mathcal O(|\mathcal{E}|)$ to $\mathcal O(n)$. 
    \item   With the help of Markov process and the reconstruction goal provided 
            by generative graph models, 
            a siamese architecture is naturally designed to transform 
            anchors into the embedding space.
    \item   We theoretically prove that a collapse, 
            where points are split into plenty of tiny groups 
            and these groups scatter randomly, 
            may happen due to the self-supervised update 
            of the graph (and connectivity distributions), 
            especially when the model is well-trained. 
            We therefore devise a specific strategy to update the graph 
            to prevent the collapse.
            The experiments also support the theoretical analysis and verify the effectiveness of the designed strategy.
\end{itemize}

\section{Preliminary and Related Work}
In this section, we introduce the existing work about graph neural 
networks and fast clustering algorithms. 

\subsection{Preliminary}
Let $\mathcal{G} = (\mathcal{V}, \mathcal{E}, \mathcal{W})$ be a graph where 
$\mathcal{V}$, $\mathcal{E}$, and $\mathcal{W}$ represent nodes, edges, 
and weights respectively. For an unweighted matrix, $\mathcal{W}$ can be 
ignored since the weights of edges are all regarded as 1. 
All matrices and vectors are represented by upper-case letters and 
lower-case letters in boldface, respectively.
$|\cdot|$ represents the size of a set. 
$(\cdot)_+ = \max(\cdot, 0)$.
For a real number $x$, $\lfloor x \rfloor$ represents the maximum integer 
that is not larger than $x$. 
$\textbf{1}_n$ denotes an $n$-dimension vector whose all entries equal 1.
In computer science, the adjacency matrix, which is denoted by $\bm A$ in this paper, 
is a common way to represent the edges of a graph.
Specifically, for an unweighted graph 
$\mathcal{G} = (\mathcal{V}, \mathcal{E})$,
$A_{ij} = 1$ if $\langle v_i, v_j \rangle \in \mathcal{E}$ and 
$A_{ij} = 0$ otherwise.
A popular model of GNN is the graph convolution network (\textit{GCN}) \cite{GCN} 
where the graph convolution layer is accordingly formulated 
as \cite{GCN} 
\begin{equation} \label{eq_gcn}
    \bm H = \varphi(\bm L \bm X \bm W),
\end{equation}
where $\varphi(\cdot)$ is an activation function, 
$\bm X \in \mathbb{R}^{n \times d}$ is the representation of data points,
$\bm W \in \mathbb{R}^{d \times d'}$ denotes the coefficients to learn, 
$\bm L = \bm D^{-\frac{1}{2}} (\bm A + \bm I) \bm D^{-\frac{1}{2}}$ is often viewed 
as a variant of the normalized Laplacian matrix, 
and $\bm D$ is the degree matrix of $(\bm A + \bm I)$. 
Specifically, $\bm D$ is a diagonal matrix where $D_{ii} = \sum _{j=1}^n (A_{ij} + 1)$.

\subsection{Graph Neural Networks}
In recent years, how to effectively extract features of graphs has attracted 
a lot of attention \cite{GCN,GNN-Survey}. 
Inspired by CNN, a few works \cite{spectralgcn,ChebNet,GCN} attempt to investigate convolution operation 
on graph-type data. 
With the rise of the attention mechanism, GAT \cite{GAT,AttentionGAE,GAT-2021} introduces self-attention 
into graph neural networks. 
Meanwhile, GNN can be applied to sequential tasks \cite{STGNN-0,STGNN-1}, which is known 
as spatial-temporal graph neural network (\textit{STGNN}). 
Some theoretical works \cite{GIN,GraphNorm,ImplicitAccGNN,Exponential,BenefitsOfDepth,OverSquashing} are also investigated in recent years. 
Some works \cite{BenefitsOfDepth} argue the impact of over-smoothing while 
literature \cite{OverSquashing} proposes the over-squashing. 
To extend GCN \cite{GCN} into unsupervised learning, graph auto-encoders \cite{GAE,AGAE,GALA,AttributedGAE,AdaGAE}, 
based on auto-encoder \cite{AE,VAE}, are developed.
Different from traditional auto-encoders, GCN \cite{GCN} is used as an encoder 
and the decoder calculates inner-products of embedding to reconstruct graphs. 
To build a symmetric architecture, GALA \cite{GALA} designs a sharpening operator 
since the graph convolution is equivalent to a smooth operator. 
SDCN \cite{SDCN} has tried to utilize GNN for clustering, while 
it just constructs a KNN graph as the preprocessing and suffers from the inefficiency problem as well. 
Although GNN has achieved outstanding performance \cite{GCN,ContrastiveGCN}, it requires too much 
memory and time. 
Some methods \cite{GraphSAGE,ClusterGCN,StoGCN,SGC,S2GC} for accelerating training and reducing consumption of memory are 
proposed. For instance, SGC \cite{SGC} attempts to accelerate by 
removing non-linear activations in GCN, which reduces the number of parameters.
More importantly, it compresses multiple layers into one layer such that 
stochastic gradient descent can be used. 
GraphSAGE \cite{GraphSAGE} attempts to sample part of nodes for each node to reduce the computational complexity, 
while FastGCN \cite{FastGCN} samples nodes independently. 
ClusterGCN \cite{ClusterGCN} tries to find an approximate graph 
with multiple connected components via METIS \cite{METIS} to employ batch gradient descent.

\subsection{Clustering on Large Scale Data}
In traditional clustering, $k$-means \cite{KMeans} works efficiently on large scale data 
but it can only deal with spherical data, which limits the application of 
$k$-means. 
By virtue of deep learning, deep clustering \cite{IMSAT,DEC,EGAE,JULE,DEPICT,ContrastiveClustering} often aims to 
learn expressive representations and then feed the representations 
into the simple clustering module. 
The whole optimization is based on batch gradient descent so that 
they are still available on large scale data. 
However, they frequently suffers from the over-fitting and the performance 
heavily relies on the tuning of neural networks. 
For graph-based clustering methods \cite{NormalizedCut,RatioCut,SC}, 
both computational complexity and space 
complexity of graph construction are $\mathcal O(n^2)$.
In particular, solving spectral clustering needs eigenvalue 
decomposition, which is also time-consuming. 
To apply spectral clustering on large scale data, a lot of 
variants \cite{LargeGraph,Nystrom,CSC,KASP} have been developed.
Anchor-based methods \cite{LargeGraph,LSC,SNC} are important extensions of 
spectral clustering.
The core idea is to construct a graph implicitly via some 
representative points, which are usually named as anchors.
By constructing a bipartite graph between samples and anchors, 
they only need to perform singular value decomposition on 
the Laplacian of the bipartite graph \cite{LSC}.
The anchor-based spectral clustering provides an elegant scheme 
to improve the efficiency of graph-based models.

\begin{myRemark}
Although several CNN-based clustering methods \cite{DEPICT,JULE,DualAE}
have achieved good performance, they focus on vision data and 
the main improvement is caused by the classical convolution operator, 
so that they could not be applied to non-Euclidean data, \textit{e.g.},
word vectors, data mining data, \textit{etc}.
Therefore, the CNN-based methods will not be used as the main competitors in this paper.
It should be emphasized that CNN-based methods and AnchorGAE should not be 
mutually exclusive. 
AnchorGAE aims to find a general paradigm to improve the clustering 
baseline under the most general settings. 
It may be an attractive topic in the future to extend the core idea of AnchorGAE 
into the vision-specific clustering by combining CNNs. 
\end{myRemark}

\section{Methodology}
In this section, we first revisit the main problems that we attempt to address in this paper.
In Section \ref{section_generative}, we show how to construct a bipartite graph on the non-graph data from a generative perspective. 
The probability perspective is vital to the derivation of the fast bipartite graph 
convolution and the reconstruction goal of the unsupervised model.
In Section \ref{section_bipartite_conv}, the core idea of AnchorGAE and its efficiency are elaborated. 
In Section \ref{section_siamese}, we will find that a siamese structure is naturally induced by the bipartite graph convolution. 
The rigorous analysis of the degeneration caused by the simple update of graphs is provided in Section \ref{section_analysis}. 
Finally, the methods to obtain the final clustering assignments are 
discussed. 
The core idea of AnchorGAE is illustrated in Figure \ref{figure_framework}.

\subsection{Revisit the Critical Problems}
Before elaborating on \textit{how to efficiently improve the capacity of 
graph-based clustering via GNNs on non-graph data}, 
we introduce the major hurdles respectively. 

\textit{\textbf{No Prior Graph:}}
In the scenarios of clustering, only features are given as the prior information. 
The graph is usually constructed as the preprocessing in graph-based clustering methods. 
The key of performance is the quality of the constructed graph. 
When the scattering of data is too complicated to precisely capture the 
key information based on the original data 
(which is the main bottleneck of the most existing graph-based clustering methods), 
it is a rational scheme to map the data into a latent space and then 
build a more precise graph in this space. 
As we focus on graph-based models, GNNs are more rational schemes compared with 
classical neural networks. 
Moreover, the utilization of the estimated graph (via GNNs) to update graph 
leads to a self-supervised procedure. 
The graph construction usually results in $\mathcal{O}(n^2)$ costs, 
which also causes the inefficiency of graph-based clustering methods. 
The scheme to solve this problem will be elaborated in Section \ref{section_generative}.

\textit{\textbf{Inefficiency of GNNs:}}
In the forward propagation of GNNs, it usually require neighbored nodes 
to compute the hidden feature of one node. 
This computation results in the inefficiency of GNNs on large graph \cite{FastGNN-Survey}, 
especially with the increase of depth.
In this paper, we simply focus on GCN, the widely used variant of GNN, 
since the theoretical analysis does not depend on specific GNN modules.
For instance, it requires $\mathcal O(|\mathcal{E}| d)$ time to compute 
$\bm L \bm X$ in the forward propagation. If $\mathcal{G}$ is sparse enough, the computational complexity 
is acceptable. If there are plenty of edges in $\mathcal{G}$, the 
computational complexity can be roughly regarded as $\mathcal O(n^2 d)$, so that 
GCNs can not be applied in this case. 
We will show how to alleviate it into $\mathcal{O}(n)$ 
in Sections \ref{section_bipartite_conv} and \ref{section_siamese}. 

\textit{\textbf{Inefficiency of Graph-Based Clustering Methods:}} 
Even after the $\mathcal{O}(n^2)$ graph construction, the computation complexity 
reaches $\mathcal{O}(n^3)$ due to the eigenvalue decomposition. 
The problem is easy to address after introducing the $\mathcal{O}(n)$ bipartite 
graph convolution and the details can be found in Section \ref{section_fast_clustering}.

\subsection{Generative Perspective for Bipartite Graphs} \label{section_generative}
By virtue of the generative perspective introduced in this section, the probability transition matrix 
is accordingly built so that the Markov process can be used to implement 
the efficient bipartite graph convolution, which is elaborated in 
Section \ref{section_bipartite_conv}. 
The generative perspective also provides a training goal, 
reconstructing the distribution, for AnchorGAE.

\subsubsection{Generative Perspective for Graphs}
For a graph $\mathcal{G} = (\mathcal{V}, \mathcal{E})$, an edge, 
$\langle v_i, v_j \rangle \in \mathcal{E}$, can be regarded as a sampling from 
an underlying connectivity distribution $p(\cdot | v_i)$ for $v_i$, 
where $v_i$ denotes the $i$-th node. 
Note that $\sum _{j=1}^{n} p(v_j | v_i) = 1$ where $n = |\mathcal{V}|$.
Accordingly, \textit{a graph can be viewed as multiple samplings regarding 
connectivity distributions $\{p(\cdot | v_i)\}_{i=1}^{n}$}.
Before the further discussion of weighted graphs, 
we further clarify the difference between the generative perspective and discriminative perspective. 
In the discriminative view, the existence of edge is certain and the model should 
train a classifier to predict $p(\langle v_i, v_j \rangle \in \mathcal{E})$ 
and $p(\langle v_i, v_j \rangle \notin \mathcal{E})$. 
Note that $\sum _{j} p(\langle v_i, v_j \rangle \in \mathcal{E}) \neq 1$ in most cases. 
In this paper, the generative view provides an elegant Bayesian framework  
to simultaneously capture the local topology and offer a reconstruction goal for GAE. 

With the generative perspective, for a weighted graph 
$\mathcal{G} = (\mathcal{V}, \mathcal{E}, \mathcal{W})$ 
where $\mathcal{W}$ represents the weights on edges,
$\mathcal{W}$ can be viewed as an estimation of 
the underlying connectivity distributions.
From the generative perspective, the underlying connectivity distributions 
can be supposed to exist in any dataset, 
while the connectivity distributions are agnostic on non-graph datasets.
However, no matter which type of data is, it will not affect the existence of connectivity distributions. 
The vital idea that leads AnchorGAE to be applied to non-graph data is to 
estimate the distributions. 
Besides, they are also the target to reconstruct by our unsupervised model.
The following assumption helps to estimate the connectivity distributions. 
\begin{myAssumption} \label{assumption_generative}
    Given an ideal metric of difference between two nodes $d(v_i, v_j)$, 
    $p(v_j | v_i)$ is negatively related to $d(v_i, v_j)$, \textit{i.e.}, 
    $p(v_j | v_i) \varpropto -d(v_i, v_j)$.
    An ideal metric indicates $\forall d(v_i, v_j) \leq d(v_i, v_l), p(v_j | v_i) \geq p(v_l | v_i)$. 
\end{myAssumption}

In this paper, each node is described by a $d$-dimension vector, \textit{i.e.},
$\bm x_i \in \mathbb{R}^{d}$. Therefore, the metric can be reformulated as 
\begin{equation} \label{eq_metric}
    d(v_i, v_j) = \|f(\bm x_i) - f(\bm x_j)\|_2^2 ,
\end{equation}
where $f(\cdot)$ represents an ideal mapping of the raw features. 
The implementation is discussed in Section \ref{section_bipartite_conv}. 

\begin{myRemark}
    Although the above assumption seems too strong to hold on raw data representation, 
    the goal of our model is to find an approximately ideal mapping function $f(\cdot)$ 
    to build a well-structured graph. 
\end{myRemark}

\subsubsection{Bipartite Graph Simplification with Anchors}
To efficiently apply graph convolution, 
we attempt to simplify the graph via some representative points, which are 
named \textit{anchors} or \textit{landmarks}. 
Let $\mathcal{U}$ denote anchors where $|\mathcal{U}| = m$ and the $i$-th 
anchor is also described by a $d$-dimension vector $\bm c_i$. 
Then, we can get a new \textit{bipartite} graph with anchors, 
$\mathcal{G}_a = (\mathcal{V}_a, \mathcal{E}_a, \mathcal{W}_a)$.
Specifically, $\mathcal{V}_a = \mathcal{V} \cup \mathcal{U}$ and 
$\mathcal{W}_a$ is the corresponding weights of $\mathcal{V}_a$.
The bipartite property means that \textbf{only edges between $\mathcal{V}$ and 
$\mathcal{U}$ are allowed in $\mathcal{E}_a$}. 
This bipartite property is from the following assumption.
\begin{myAssumption}
    Given an ideal set of anchors $\mathcal{U}$, $p(v_j | v_i)$ can be constructed by $\{p(u_l | v_i)\}_{l=1}^m$ and 
    $\{p(v_j | u_l)\}_{l=1}^m$.
\end{myAssumption}

Since both connectivity distributions and anchors are unknown in 
general data, we need to estimate them.
First, we intend to estimate anchors and $p(u_j | v_i)$ alternatively by 
\begin{equation}
    \min \limits_{\mathcal{U}, p(\cdot | v_i)} \sum \limits_{i=1}^n \mathbb{E}_{u_j \sim p( \cdot | v_i)} d(v_i, u_j) ,
\end{equation}
according to Assumption \ref{assumption_generative}. 
Specifically speaking, in the ideal metric space introduced by Assumption \ref{assumption_generative}, 
the expected divergence between $\mathcal{V}$ and $\mathcal{U}$ should be minimum. 
Note that we only focus on the estimation of $p(u_j | v_i)$ and overlook 
$p(v_i | u_j)$ for now. 
However, the above problem has a trivial solution, \textit{i.e.}, 
\begin{equation}
    p_\dag(u_j | v_i) = 
    \left \{ 
    \begin{array}{l l}
        1, & j = \arg \min_j (d(v_i, u_j)), \\
        0, & \text{else} .
    \end{array}
    \right .
\end{equation}
It indicates that each original node (from $\mathcal{V}$) only connects with 
its nearest anchors (from $\mathcal{U}$). 
To avoid the ill connectivity distributions and anchor graph, we turn to the following regularized problem,
\begin{equation}
    \min \limits_{\mathcal{U}, p(\cdot | v_i)} \sum \limits_{i=1}^n \mathbb{E}_{u_j \sim p( \cdot | v_i)} d(v_i, u_j) 
    + \ell(p(\cdot | v_i), \pi(\cdot | v_i)),
\end{equation}
where $\pi(\cdot | v_i)$ represents the uniform distribution and $\ell(\cdot, \cdot)$
is a metric of two discrete distributions. 
Intuitively, $p(u_j | v_i)$ should be sparse such that anchors far 
from $v_i$ are ignored. 
In practice, Kullback-Leibler divergence (\textit{KL}-divergence) is usually used as 
$\ell(\cdot, \cdot)$ to measure the divergence of two probability distributions.
However, it will return a \textit{dense} solution of $p(\cdot | v_i)$. 
Instead, we employ a simple measure as 
\begin{equation} \label{obj_anchor}
    \min \limits_{\mathcal{U}, p(\cdot | v_i)} \sum \limits_{i=1}^n \mathbb{E}_{u_j \sim p( \cdot | v_i)} d(v_i, u_j) 
    + \gamma_i \sum \limits_{j=1}^m (p(u_j | v_i) - \pi(u_j | v_i))^2 .
\end{equation}
Let $d_{k} (v_i, \cdot)$ be the $k$-th smallest value of $\{d(v_i, u_j)\}_{j=1}^m$, and if
$\gamma_i = \frac{1}{2} (k d_{k+1}(v_i, \cdot) - \sum _{l=1}^k d_l (v_i, \cdot))$, 
$p(\cdot|v_i)$ will be $k$-sparse (only $k$ entries are non-zero). 
Meanwhile, it can be transformed into \cite{CAN} and solved by
\begin{equation} \label{eq_p}
    p(u_j | v_i) = (\frac{d_{k+1}(v_i, \cdot) - d(v_i, u_j)}{\sum _{l=1}^k d_{k+1}(v_i, \cdot) - d_l(v_i, \cdot)} )_+ .
\end{equation}
In other words, the hyper-parameter $\gamma_i$ is converted 
to the sparsity $k$, which is much easier to tune.

\begin{algorithm}[t]
    \centering
    \caption{Iterative update of the connectivity distribution and anchors, \textit{i.e.,}, optimization of problem (\ref{obj_anchor}).}
    \label{alg_anchor}
    \begin{algorithmic}
        \REQUIRE $\{f(\bm x_i)\}_{i=1}^n$, $\{f^{(0)}(\bm c_j)\}_{j=1}^m$, sparsity $k$.
        \STATE $t = 0$.
        \STATE $d^{(0)}(v_i, u_j) = \|f(\bm x_i) - f^{(0)}(\bm u_j)\|_2^2$.
        \REPEAT
        \STATE $t = t + 1$.
        \STATE $p^{(t)}(u_j | v_i) = (\frac{d^{(t-1)}_{k+1}(v_i, \cdot) - d^{(t-1)}(v_i, u_j)}{\sum _{l=1}^k d^{(t-1)}_{k+1}(v_i, \cdot) - d_l^{(t-1)}(v_i, \cdot)} )_+$.
        \STATE $f^{(t)}(\bm c_j) = \frac{\sum _{i=1}^n p^{(t)}(u_j | v_i) f(\bm x_i)}{\sum _{i=1}^n p^{(t)}(u_j | v_i)}$.
        \STATE $d^{(t)}(v_i, u_j) = \|f(\bm x_i) - f^{(t)}(\bm c_j)\|_2^2.$
        \UNTIL{convergence .}
        \STATE $f(\bm c_j) = f^{(t)}(\bm c_j)$, $p(\cdot | v_i) = p^{(t)}(\cdot | v_i)$. 
        \ENSURE $f(\bm c_j)$ and $p(\cdot | v_i)$ .
    \end{algorithmic}
\end{algorithm}

When $p(u_j | v_i)$ is fixed, anchors are computed by solving the following problem,
\begin{equation}
    \min \limits_{f(\bm c_j)} \mathbb E_{u_j \sim p(\cdot | v_i)} \|f(\bm x_i) - f(\bm c_j)\|_2^2 ,
\end{equation}
where $\bm c_j$ is representation of anchor $u_j$ in the original space. 
If we take the derivative of the above equation regarding $f(\bm c_j)$, 
\begin{equation}
    \begin{split}
                & \nabla_{f(\bm c_j)} \mathbb{E}_{u_j \sim p(\cdot | v_i)} \|f(\bm x_i) - f(\bm c_j)\|_2^2 \\
                = ~ & 2(\sum \limits_{i=1}^m p(u_j | v_i) f(\bm c_j) - p(u_j | v_i) f(\bm x_i) ) ,
    \end{split}
\end{equation}
and set it to 0, then we have
\begin{equation} \label{eq_anchor}
        f(\bm c_j) = \frac{\sum _{i=1}^n p(u_j | v_i) f(\bm x_i)}{\sum _{i=1}^n p(u_j | v_i)} .
\end{equation}
In sum, problem (\ref{obj_anchor}) can be solved iteratively 
through Eq. (\ref{eq_p}) and (\ref{eq_anchor}), 
which is summarized as Algorithm \ref{alg_anchor}. 

\textit{\textbf{Build A Generative Bipartite Graph by $\{p(u_j|v_i)\}_{i,j}$ and $\{f(\bm c_j)\}_j$}}:
Then we turn to model $p(v_i | u_j)$, 
which is used to compute $p(v_j | v_i)$, 
the probability between two raw nodes. 
Rather than solving a problem like problem (\ref{obj_anchor}), 
$p(v_i | u_j)$ is set by a simple normalized step, 
\begin{equation} \label{eq_p_vu}
    p(v_i | u_j) = \frac{p(u_j | v_i)}{\sum _{l=1}^{n} p(u_j | v_l)} ,
\end{equation}
via utilizing $p(u_j | v_i)$ calculated by Eq. (\ref{eq_p}). 
To show how the above formulations construct a bipartite graph and 
simplify the succeeding discussion about the fast convolution, 
we further reformulate $p(u_j | v_i)$ and $p(v_i | u_j)$ 
by matrix form. 
Let $\bm B \in \mathbb{R}^{n \times m}$ be a matrix where $b_{ij} = p(u_j | v_i)$.
Then 
\begin{equation}\label{eq_transfer_prob}
    \bm T = 
    \left [
    \begin{array}{c c}
        \bm 0 & \bm B \\
        \bm B^T & \bm 0 
    \end{array}
    \right ]
    ~ \text{and} ~
    \bm D_a = 
    \left [
    \begin{array}{c c}
        \bm I & \bm 0 \\
        \bm 0 & \bm \Delta \\
    \end{array}
    \right ] 
\end{equation}
represent an \textit{unnormalized} bipartite graph and its degree matrix, 
respectively. 
Note that $\bm T$ does not conform to the generative perspective for graphs, 
though it is a bipartite graph.
According to Eq. (\ref{eq_p_vu}), the \textit{generative bipartite graph} 
can be defined as 
\begin{equation}
    \bm P = \bm D_a^{-1} \bm T = 
    \left [
    \begin{array}{c c}
        \bm 0 & \bm B \\
        \bm \Delta^{-1} \bm B^T & \bm 0 
    \end{array}
    \right ]
    \in \mathbb{R}^{(n + m) \times (n + m)} ,
\end{equation}
which can be also regarded as a probability transferring matrix. 

\begin{algorithm}[t]
    \centering
    \caption{Optimization of AnchorGAE.}
    \label{alg_anchor_gae}
    \begin{algorithmic}[1]
        \REQUIRE Raw features $\bm X$, initial sparsity $k_0$, 
        the estimated upper-bound of sparsity $k_m$, maximum epochs $E$. 
        \STATE Initialize $\bm B$ and $\bm C$ via solving problem (\ref{obj_anchor}).
        \STATE Initialize $\bm A_a$ by Eq. (\ref{eq_adjacency}).
        \STATE $k \leftarrow k_0$.
        \FOR {$i = 1, 2, \cdots, E$}{
        \STATE \textit{$\#$ Obtain the mapping $f(\cdot)$.}
        \STATE Update parameters of GCN layers by gradient descent. 
        \STATE Update $B$ and $f(\bm c_j)$ on $\bm Z$ and $\bm Z_t$ via Algorithm \ref{alg_anchor}.
        \STATE Update anchors in the original space for the bipartite graph convolution: $\bm c_j = f^{-1}(f(\bm c_j)) = \bm X^T \bm b_j$.
        \STATE Update $\bm A_a$ by Eq. (\ref{eq_adjacency}). 
        \STATE $k \leftarrow k + \Delta k$.
        }\ENDFOR
        \STATE Perform fast clustering on $\bm Z$.
        \ENSURE Clustering assignments and embedding $\bm Z$.
    \end{algorithmic}
\end{algorithm}

\subsection{Efficient $\mathcal O(n)$ Bipartite Graph Convolution} \label{section_bipartite_conv}

As the goal is to capture the structure information, 
\textbf{a multi-layer GCN is a rational scheme to implement $f(\bm x_i)$} 
which is defined in Eq. (\ref{eq_metric}). 
In this subsection, we will show why the bipartite graph convolution 
accelerates GCN. 

We utilize the bipartite graph, defined in Eq. (\ref{eq_transfer_prob}), 
to accelerate the graph convolution.
In virtue of the Markov process \cite{RandomWalk}, $p(v_j | v_i)$ can be 
obtained by the one-step transition probability, 
\textit{i.e.}, 
\begin{equation}
p(v_j | v_i) = \sum _{l=1}^{m}p(v_j | u_l) p(u_l | v_i). 
\end{equation}
Similarly, $p(u_j | u_i) = \sum _{l=1}^{n}p(u_j | v_l) p(v_l | u_i)$. 
Formally, the constructed graphs of $\mathcal{V}$ and $\mathcal{U}$ are 
defined as 
\begin{equation} \label{eq_adjacency}
    \bm A_a = \bm P^2 = 
    \left [
    \begin{array}{c c}
        \bm B \bm \Delta^{-1} \bm B^T & \bm 0 \\
        \bm 0 & \bm \Delta^{-1} \bm B^T \bm B 
    \end{array}
    \right ] = 
    \left [
    \begin{array}{c c}
        \bm A & \bm 0 \\
        \bm 0 & \bm A_t
    \end{array}
    \right ]
    .
\end{equation}
Accordingly, $\bm A$ is the weighted adjacency matrix, constructed by anchors, 
of $n$ samples, and $\bm A_t$ is the weighted adjacency of $m$ anchors. 
\begin{myRemark}
    In GCNs, the graph is conventionally preprocessed by adding self-loops as 
    shown in Eq. (\ref{eq_gcn}). The graph constructed by Eq. 
    (\ref{eq_adjacency}) is equivalent to be assigned with adaptive self-loops, 
    since $\bm A_{ii}$ and $(\bm A_t)_{ii}$ are non-zero in most cases. 
\end{myRemark}

Therefore, $\bm A$ can be directly used as a part of the graph convolution operator. 
Since 
\begin{equation}
    \bm A \textbf{1} = \bm B \bm \Delta^{-1} \bm B^T \textbf{1} = \bm B \textbf{1} = \textbf{1},
\end{equation}
the degree matrix of $\bm A$ is the identity matrix, \textit{i.e.}, 
$\bm L = \bm A$. 
Therefore, the output of GCN is formulated as 
\begin{equation} \label{eq_anchor_conv}
    \bm H = \varphi(\bm B \bm \Delta^{-1} \bm B^T \bm X \bm W) .
\end{equation}
Here, we explain why the above equation accelerates the operation. It should 
be emphasized that $\bm A$ is not required to be computed explicitly. The core 
is to \textbf{rearrange the order of matrix multiplications} as follows
\begin{equation}
    \begin{split}
    & \bm T_1 = \bm B^T \bm X \\
    \Rightarrow & ~ \bm T_2 = \bm \Delta^{-1} \bm T_1 \\
    \Rightarrow & ~ \bm T_3 = \bm T_2 \bm W \\
    \Rightarrow & ~ \bm H = \varphi(\bm B \bm T_3) .
    \end{split}
\end{equation}
According to the above order of computation, the computational complexity 
is reduced to $\mathcal O(nmd + m^2 d + m d d' + nmd')$. Since $d$, $d'$, and $m$ are usually 
much smaller than $n$, the complexity can be regarded as $\mathcal O(n m)$. 
Moreover, if the amount of anchors is small enough, the required time of 
each forward propagation of a GCN layer is $\mathcal O(n)$.

\subsection{Siamese Architecture for Anchor Mapping} \label{section_siamese}
After showing the accelerated GCN mapping based on anchors and the bipartite graph, 
we discuss more details of the whole procedure in this section, 
especially the changes of architecture brought by anchors. 

As the task is unsupervised, 
the whole architecture is based on graph auto-encoder (\textit{GAE}).
The encoder consists of multiple bipartite GCN layers defined 
in Eq. (\ref{eq_anchor_conv}), 
\begin{equation}
    \bm Z = f(\bm X) = \varphi_{L}(\bm A \varphi_{L-1}(\cdots \varphi_1(\bm A \bm X \bm W_1)\cdots) \bm W_L),
\end{equation}
where $L$ denotes the number of layers.
The embedding (denoted by $\bm Z$), which is produced by the encoder, is a non-linear mapping for the 
raw representation of $\mathcal{V}$.

However, the mentioned encoder just transforms $\mathcal{V}$ 
and there is no edge between $\mathcal{V}$ and $\mathcal{U}$. 
In other words, \textit{$\mathcal{U}$ can not be projected via the network with 
the adjacency $\bm A$}. 
Meanwhile, GCN layers suffer from the \textit{out-of-sample} problem. 
Accordingly, $\mathcal{U}$ can not be directly projected into the deep 
feature space by the encoder. 
To map anchors into the same feature space, a siamese \cite{SiameseNetwork} encoder is thus designed.
The encoder for anchors should share the same parameters but \textbf{has its own 
graph structure}. Surprisingly, the one-step probability transition matrix 
defined in Eq. (\ref{eq_adjacency}) gives us a graph that is only composed 
of anchors, $\bm A_t = \bm \Delta^{-1} \bm B^T \bm B$. As 
\begin{equation}
    \bm A_t \textbf{1} = \bm \Delta^{-1} \bm B^T \bm B \textbf{1} = \bm \Delta^{-1} \bm B^T \textbf{1} = \textbf{1},
\end{equation}
the corresponding degree matrix, $\bm D_t$, is also the identity matrix, 
\textit{i.e.}, 
\begin{equation}
    \bm L_t = \bm D_t^{-\frac{1}{2}} \bm A_t \bm D_t^{-\frac{1}{2}} = \bm A_t.
\end{equation}
Accordingly, the embedding of anchors is formulated as 
\begin{equation}
    \bm Z_t = f(\bm C) = \varphi_{L}(\bm A_t \varphi_{L-1}(\cdots \varphi_1(\bm A_t \bm C \bm W_1)\cdots) \bm W_L),
\end{equation}
where $\bm C \in \mathbb{R}^{m \times d}$ represents $m$ anchors.

\textbf{Loss}: 
As the underlying connectivity distributions have been estimated, we use them, 
rather than the adjacency in the existing GAEs,
as the target for the decoder to rebuild. 
According to Assumption \ref{assumption_generative}, the connectivity
probability should be reconstructed according to Euclidean distances between 
$\mathcal{V}$ and $\mathcal{U}$. Formally speaking, the decoder is 
\begin{equation}
    \begin{split}
        q(u_j | v_i) & = \frac{\exp(-d(v_i, u_j))}{\sum _{l=1}^m \exp(-d(v_i, u_l))} .\\
    \end{split}
\end{equation}
Instead of using mean-square error (\textit{MSE}) to reconstruct $A$, 
AnchorGAE intends to reconstruct the underlying connectivity distribution 
and employ the cross-entropy to measure the divergence,
\begin{equation}
    \mathcal{L} = \sum \limits_{i=1}^n \sum \limits_{j=1}^m p(u_j | v_i) \log \frac{1}{q(u_j | v_i)} .
\end{equation}
The GCN part of AnchorGAE is trained via minimizing $\mathcal{L}$.

So far, how to get anchors at the beginning is ignored. In deep learning, 
a common manner is to pre-train a neural network. 
Nevertheless, the pre-training of raw graph auto-encoders  
will result in nearly $\mathcal O(n^2)$ time, which violates the motivation of our model. 
Therefore, we first compute anchors and transition probability on the raw 
features. Then \textbf{anchors and transition probability will be rectified dynamically}, 
which will be elaborated in Section \ref{section_analysis}. 

\subsubsection{Why not to reconstruct features?}
One may ask why AnchorGAE just reconstruct the constructed graph rather than 
the features, which is a priori. 
The goal of the classical auto-encoder is to restore the input features. 
However, the learned representations are often not promising enough 
due to \textit{no limitation on the latent space}. 
Accordingly, several works (\textit{e.g.}, adversarial auto-encoder \cite{AdversarialAE}, 
variational auto-encoder \cite{VAE}) constrain the representations. 
More importantly, 
contrastive learning \cite{SimCLR} does not reconstruct the features anymore 
and focuses on how to pull similar samples together 
and push dissimilar samples away. 
Contrastive learning has achieved impressive performance in recent years.

In AnchorGAE, the graph is constructed on the learned representations. 
Define a mapping 
$\psi: \mathbb R^{n \times d} \rightarrow \mathbb{R}^{n \times d}$. 
Then the graph can be roughly formulated as $\bm A = \psi(f(\bm X))$ 
where $f$ denotes the GNN mapping. 
On the one hand, the composite mapping, $\psi \circ f$, 
removes the redundant information of $X$. 
\textit{The reconstruction of $\bm A$ can be regarded as the restoration 
of the simplified $\bm X$}. 
On the other hand, the edges of graphs indicate similar sample pairs and dissimilar pairs. 
\textit{The reconstruction of graph is therefore consistent with the idea of contrastive learning}. 
That's why AnchorGAE only reconstructs the constructed graph.

\begin{figure}[t]
    \small
    \centering
    \setlength{\fboxrule}{0.1pt}
    \setlength{\fboxsep}{0.5pt}
    \subcaptionbox{Collapse}{
        \fbox{
        \includegraphics[width=0.45\linewidth]{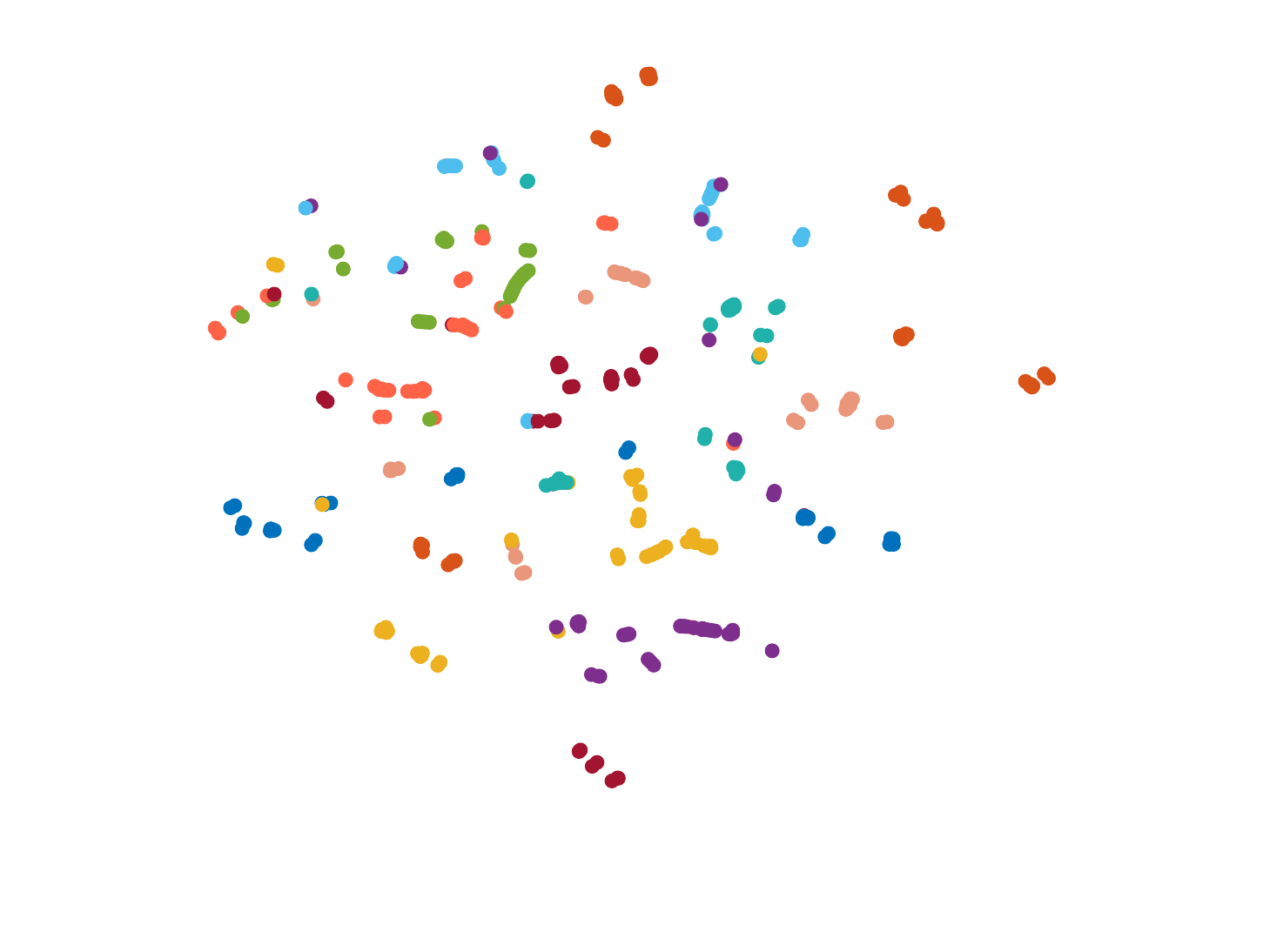}
        }
    }
    \subcaptionbox{AnchorGAE}{
        \fbox{
        \includegraphics[width=0.45\linewidth]{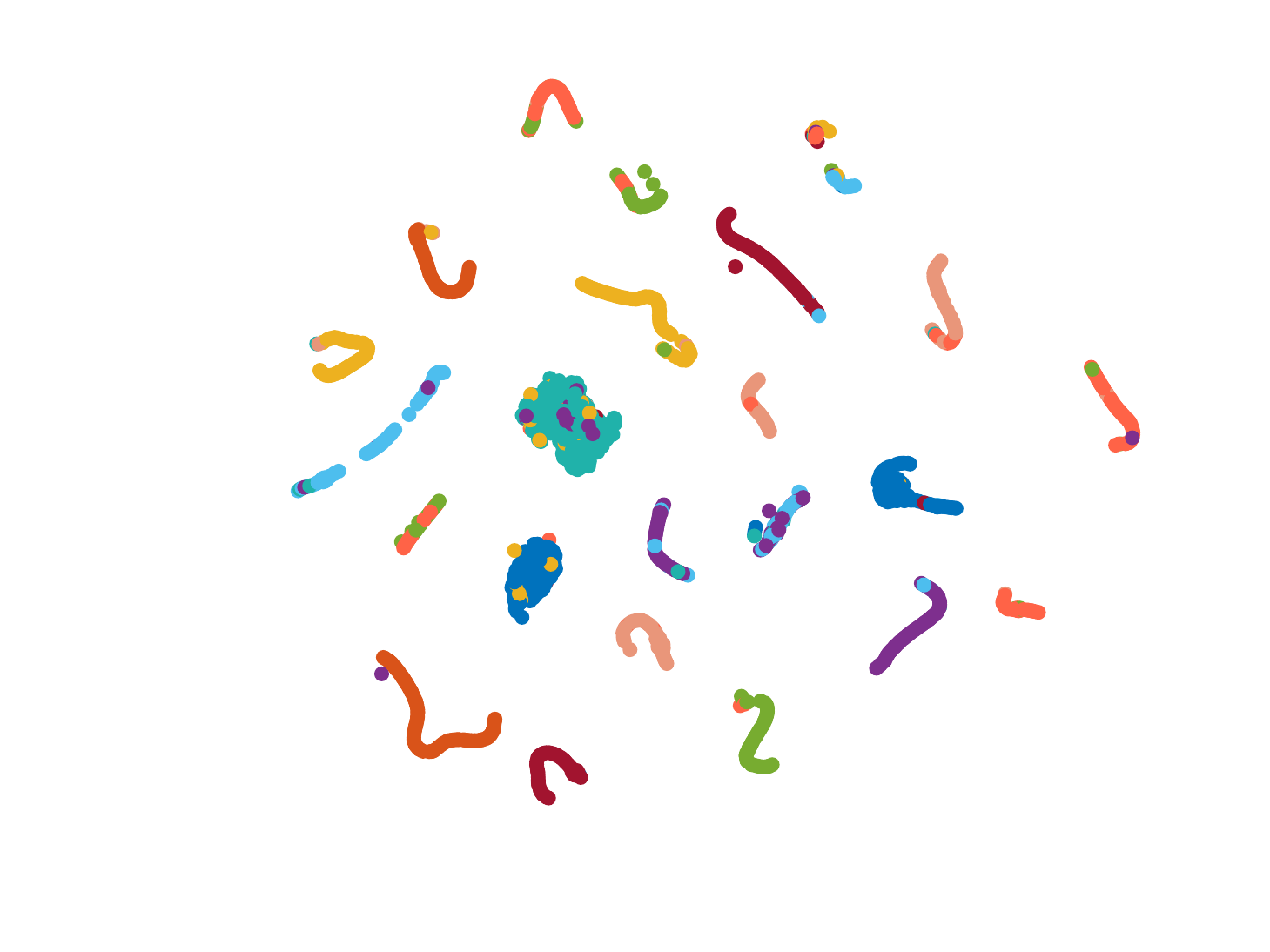}
        }
    }
    
    \caption{Illustration of the strategy to avoid degeneration on MNIST-test: 
    The left one is from AnchorGAE with fixed $k$ while the right one is 
    from AnchorGAE with dynamically increasing $k$.}
    \label{figure_collapse}
\end{figure}

\begin{table*}[t]
    \centering
    \renewcommand\arraystretch{1.2}
    \setlength{\tabcolsep}{3mm}
    \caption{Experimental Results: ACC and NMI}
    \label{table_results}
    \begin{tabular}{l c c c c c c c c c c c c c c}
    \hline
    
    \hline
        \multirow{2}{*}{Method} & \multicolumn{2}{c}{ISOLET} & \multicolumn{2}{c}{SEGMENT} & \multicolumn{2}{c}{USPS} & \multicolumn{2}{c}{MNIST-test} & \multicolumn{2}{c}{MNIST-full} & \multicolumn{2}{c}{Fashion} \\
        & ACC & NMI & ACC & NMI & ACC & NMI& ACC & NMI & ACC & NMI & ACC & NMI \\
    \hline
    \hline
        
        K-Means & 0.621 & 0.775 & 0.559 & 0.590 & 0.648 & 0.628 & 0.548 & 0.501 & 0.534 & 0.500 & 0.474 & 0.512 \\
        Nystrom & 0.648 & 0.780 & 0.562 & 0.549 & 0.668 & 0.619 & 0.575 & 0.505 & 0.589 & 0.502 & 0.563 & 0.522 \\
        CSC     & \underline{0.661} & 0.785 & \underline{0.606} & 0.555 & 0.726 & 0.696 & 0.623 & 0.572 & 0.592 & 0.569 & 0.515 & 0.501 \\
        KASP    & 0.630 & 0.754 & 0.554 & 0.544 & 0.712 & 0.693 & 0.654 & 0.600 & 0.620 & 0.555 & 0.538 & 0.529\\
        LSC     & 0.657 & 0.774 & 0.562 & 0.557 & 0.727 & 0.694 & 0.702 & 0.615 & 0.714 & 0.623 & \underline{0.634} & \underline{0.629} \\
        SNC     & 0.553 & 0.631 & 0.567 & 0.535 & 0.703 & 0.589 & 0.570 & 0.523 & --- & --- & --- & ---\\
        SGC     & 0.464 & 0.665 & 0.524 & \underline{0.611}& 0.657 & 0.801 & 0.719 & 0.757 & --- & --- & --- & --- \\
        GAE     & 0.612 & 0.783 & 0.442 & 0.475 & 0.679 & \underline{0.821} & 0.696 & 0.758 & --- & --- & --- & --- \\
        DEC     & 0.425 & 0.685 & 0.143 & 0.000 & 0.621 & 0.588 & 0.813 & 0.803 & 0.797 & \textbf{0.810} & 0.516 & 0.541 \\
        SpectralNet & 0.528 & 0.781 & 0.515 & 0.540 & 0.678 & 0.818 & 0.820 & 0.817 & --- & --- & --- & --- \\
        ClusterGCN-L2 & 0.581 & 0.772 & 0.580 & 0.601 & 0.659 & 0.769 & 0.632 & 0.666 & 0.707 & \underline{0.730} & 0.532 & 0.605 \\
        ClusterGCN-L4 & 0.574 & 0.738 & 0.593 & 0.599 & 0.662 & 0.769 & 0.631 & 0.665 & 0.708 & 0.732 & 0.532 & 0.606 \\
    \hline
        Ours-A (fixed $B$) & 0.615 & \textbf{0.814} & 0.443 & 0.490 & 0.490 & 0.475 & 0.290 & 0.351 & 0.468 & 0.635 & 0.588 & \textbf{0.630} \\
        Ours-B (fixed $k$) & 0.373 & 0.532 & 0.285 & 0.125 & 0.265 & 0.174 & 0.284 & 0.228 & 0.289 & 0.232 & 0.194 & 0.112 \\
        Ours-C (KNN) & 0.654 & 0.778 & 0.555 & 0.511 & 0.729 & 0.746 & 0.776 & 0.735 & 0.793 & 0.743 & 0.533 & 0.553 \\
        AnchorGAE-L2 & \textbf{0.663} & \underline{0.787} & \textbf{0.635} & \textbf{0.613} & \textbf{0.853} & \textbf{0.828} & \textbf{0.823} & \textbf{0.773} & \textbf{0.833} & \underline{0.773} & \textbf{0.645} & 0.607 \\
        AnchorGAE-L4 & 0.591 & 0.756 & 0.576 & 0.508& \underline{0.781} & 0.775 & \underline{0.810} & \underline{0.762} & \underline{0.808} & 0.771 & 0.613 & 0.590 \\
    \hline
    
    \hline
    \end{tabular}
    
\end{table*}

\subsection{Will the Adaptive Update Work?} \label{section_analysis}
Since the initialized anchors and transition probability is calculated from 
the original features, the information hidden in $\bm A$ may be limited 
especially when the data distribution is complicated. 
To precisely exploit the high-level information in a self-supervised way, 
a feasible method is to adaptively update anchors according to the learned embedding. 

To achieve this goal, we first need to recompute anchors, $\mathcal{U}$. 
According to Eq. (\ref{eq_p}) and (\ref{eq_anchor}), we can obtain anchors under the deep representations, $\{f(\bm c_j)\}_{j=1}^m$.
However, \textbf{the input of our model requires anchors under the raw features}. 
In other words, it is inevitable to compute $\bm c_j = f^{-1}(f(\bm c_j))$.
Unfortunately, it is hard to ensure $f(\cdot)$ an invertible mapping. 
A scheme is to utilize the classical auto-encoders since the decoder that 
tries to reconstruct the raw input from the embedding can be 
regarded as an approximation of the inverse mapping of the encoder. 
To simplify the discussion in this paper, 
we simply set 
\begin{equation}
    \bm c_j = f^{-1}(f(\bm c_j)) \approx \bm X^T \bm b_j,
\end{equation}
to estimate $\bm c_j$ under the original feature space. 

In our expectations, the self-supervised update of anchors and transition probability 
would become more and more precise. 
However, the performance of AnchorGAE will become worse and worse after several updates 
of anchors, which has been proved theoretically and empirically. 
An interesting phenomenon is that \textbf{a better reconstruction 
indicates a more severe collapse}. 
The following theorem provides a rigorous description of this collapse. 
\begin{myTheorem} \label{theo}
    Let $p_{k}(\cdot | v_i)$ be the $k$-largest one of $\{p(u_j | v_i)\}_{j=1}^m$ and 
    $\{\hat p(u_j | v_i)\}_{j=1}^m$ be the updated distribution with the same 
    $k$ after step 6 of Algorithm \ref{alg_anchor_gae}. 
    If $|q(\cdot | v_i) - p(\cdot | v_i)| \leq \varepsilon$, 
    then there exists a constant $\delta > 0$,
    so that when $\varepsilon \leq \delta$, 
    $$|\hat p_j(\cdot | v_i) - \frac{1}{k}| \leq \mathcal{O}(\log^{-1/2}(1/\varepsilon)) $$ 
    holds for any $j \leq k$. 
\end{myTheorem}
The above conclusion
means that $\hat p(\cdot | v_i)$ is a sparse and uniform distribution.
In other words, the constructed graph \textbf{degenerates into an unweighted graph}, 
which is usually unexpected in clustering. 
More intuitively, the degeneration is caused by \textit{projecting 
samples from the identical cluster into multiple tiny and cohesive clusters} 
as shown in Figure \ref{figure_collapse}.
These tiny clusters have no connection with each other and they scatter 
randomly in the representation space, such that the clustering result is 
disturbed. 
To avoid this collapse, we propose two strategies: (1) increase $k$ which represents 
the sparsity of neighbors; (2) decrease the number of anchors, $m$.
In this paper, we dynamically increase $k$, 
\textit{which is equivalent to attempt to connect these tiny groups before they scatter randomly}, as follows 
\begin{equation}
    k \leftarrow k + \Delta k .
\end{equation}
Suppose that there are $c$ clusters and 
the number of samples in the smallest cluster is represented by $n_s$. 
Then we can define the upper-bound of $k$ and increment of sparsity $\Delta k$ as 
\begin{equation}
    k_m = \lfloor m \times \frac{n_s}{n}\rfloor  ~~ \textrm{and} ~~ \Delta k = \lfloor \frac{k_m - k_0}{E} \rfloor, 
\end{equation}
where $k_0$ denotes the sparsity of the initialized anchor graph and $E$ be 
the number of iterations to update anchors. 
If we have no prior information of $n_s$, $n_s$ can be simply set as 
$\lfloor n/c \rfloor$ or $\lfloor n / (2c) \rfloor$.
The brief modification helps a lot to guarantee the stability and performance 
of AnchorGAE.

\begin{figure*}[t]
    \small
    \centering
    \setlength{\fboxrule}{0.1pt}
    \setlength{\fboxsep}{1pt}
    \subcaptionbox{Original features}{
        \fbox{
        \includegraphics[width=0.21\linewidth]{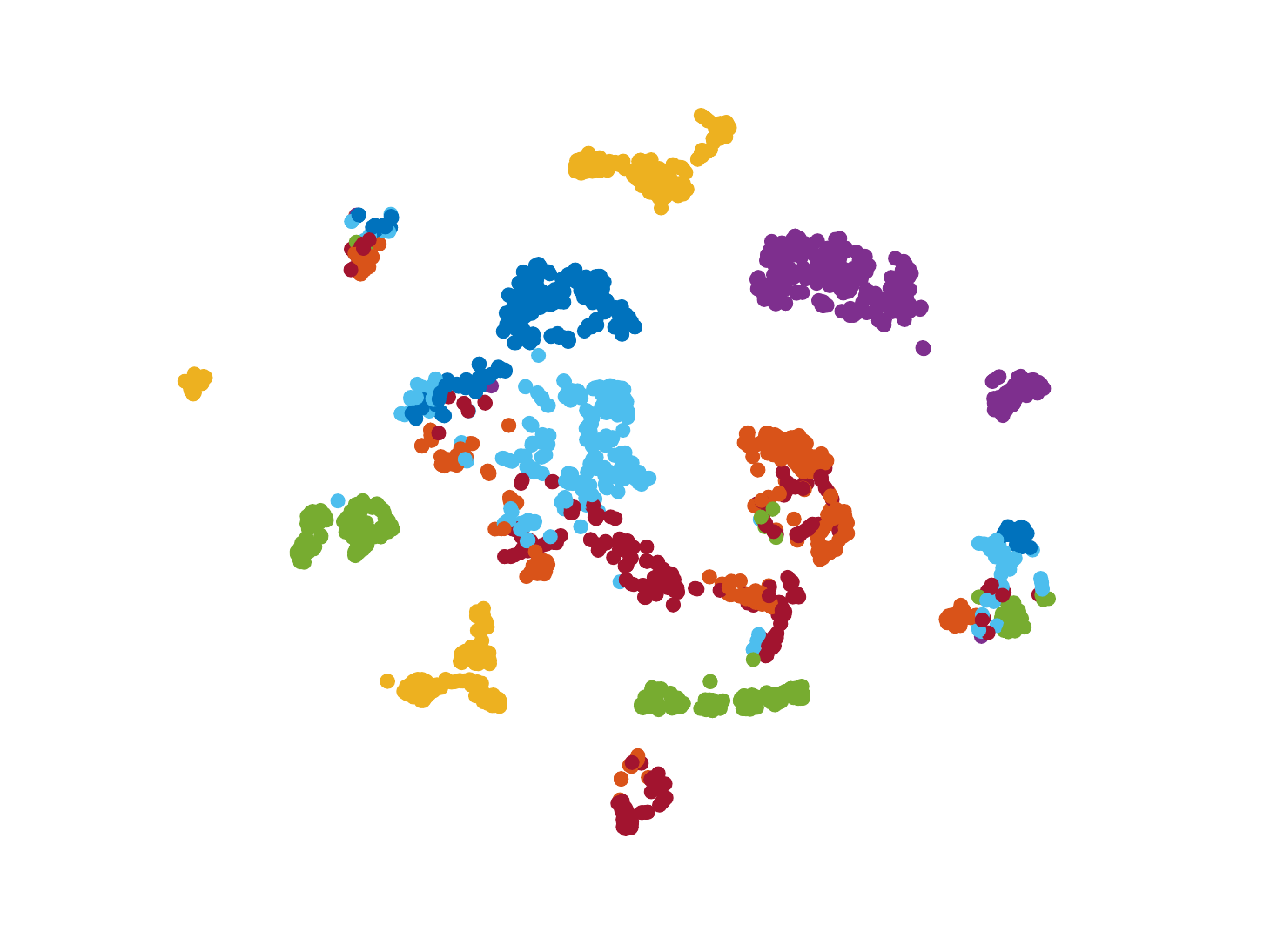}
        }
    }
    \subcaptionbox{Fixed $k$}{
        \label{subfigure_fixed_k}
        \fbox{
        \includegraphics[width=0.21\linewidth]{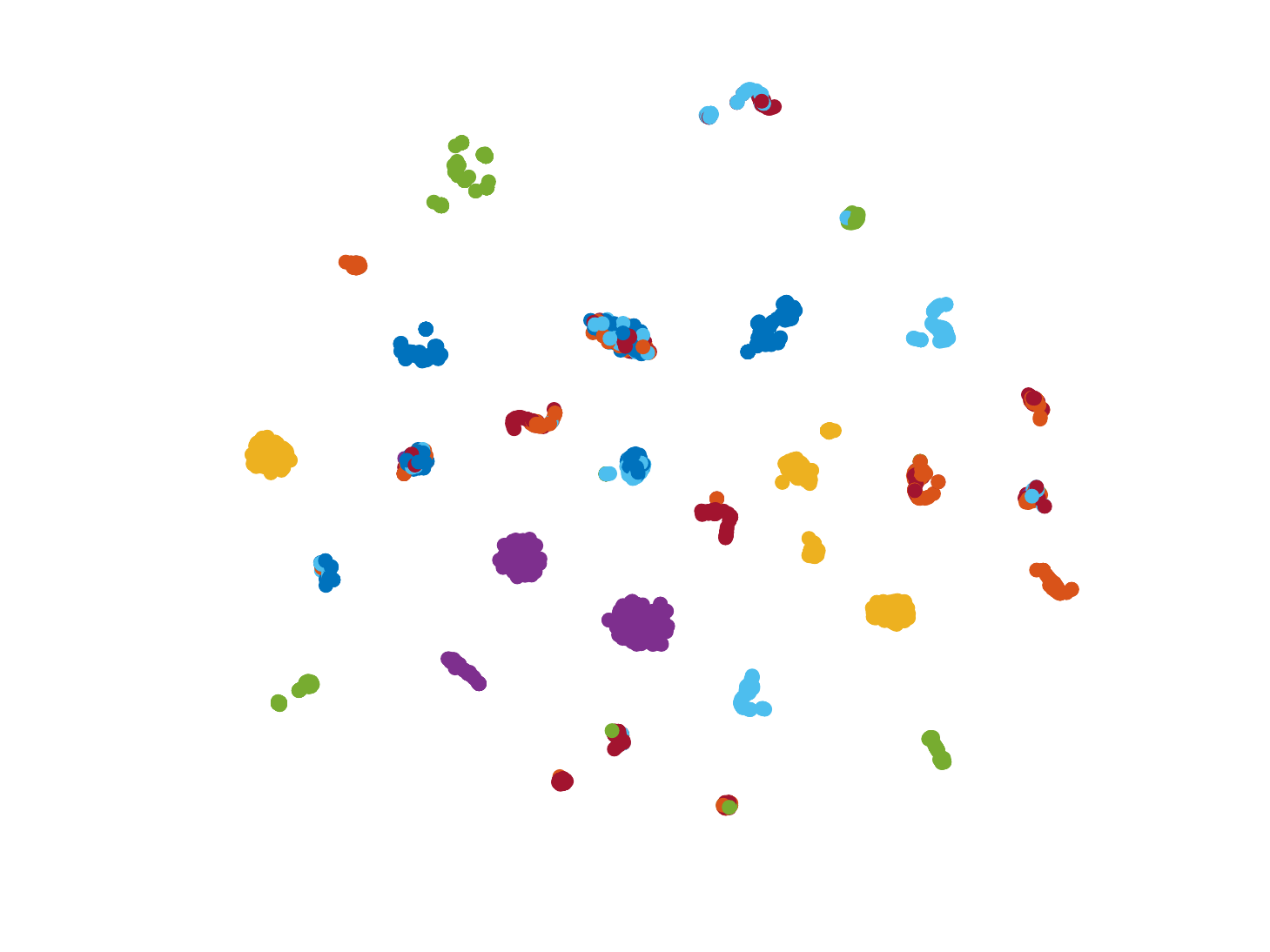}
        }
    }
    \subcaptionbox{Fixed $B$}{
        \fbox{
        \includegraphics[width=0.21\linewidth]{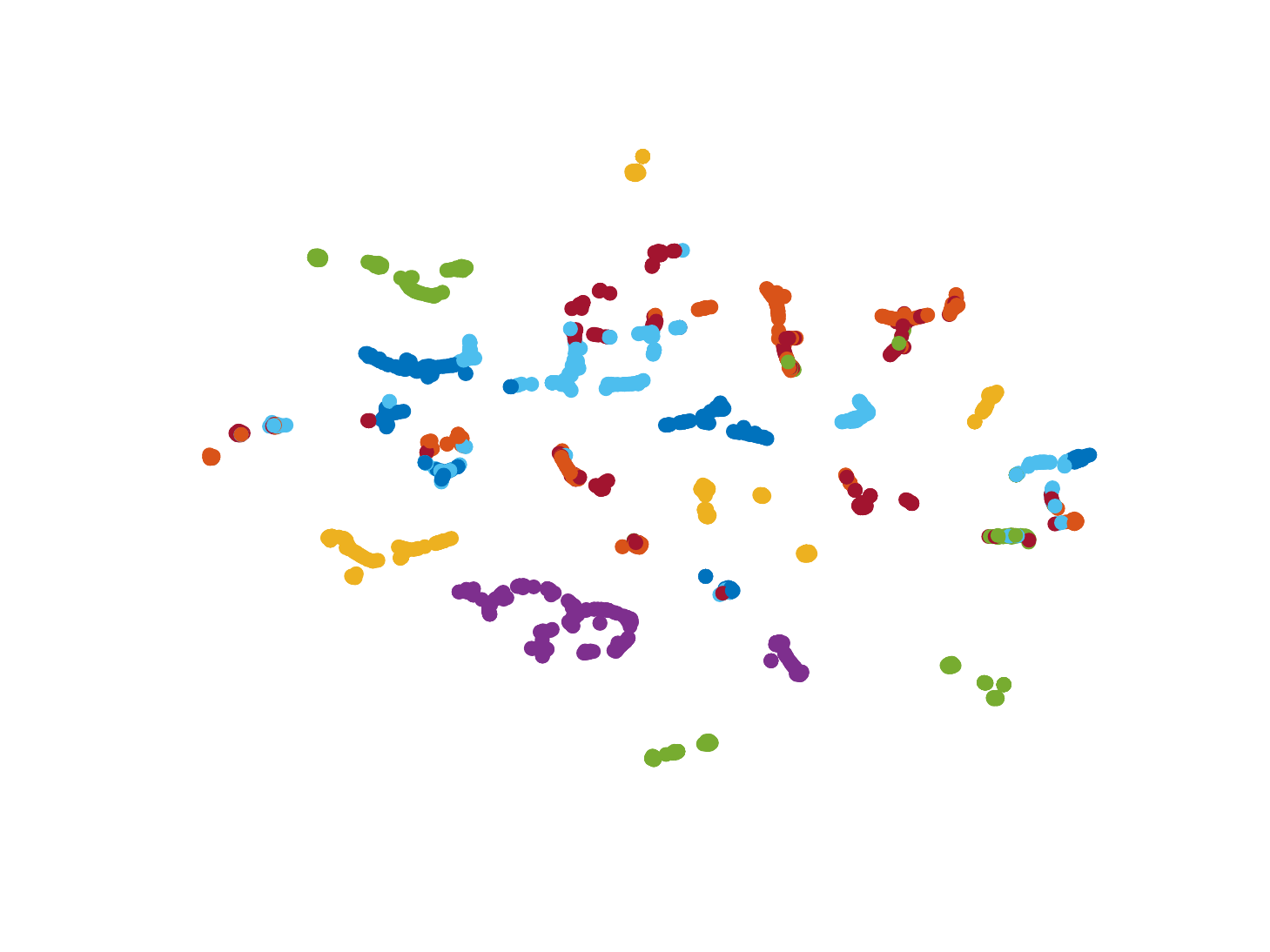}
        }
    }
    \subcaptionbox{AnchorGAE}{
        \fbox{
        \includegraphics[width=0.21\linewidth]{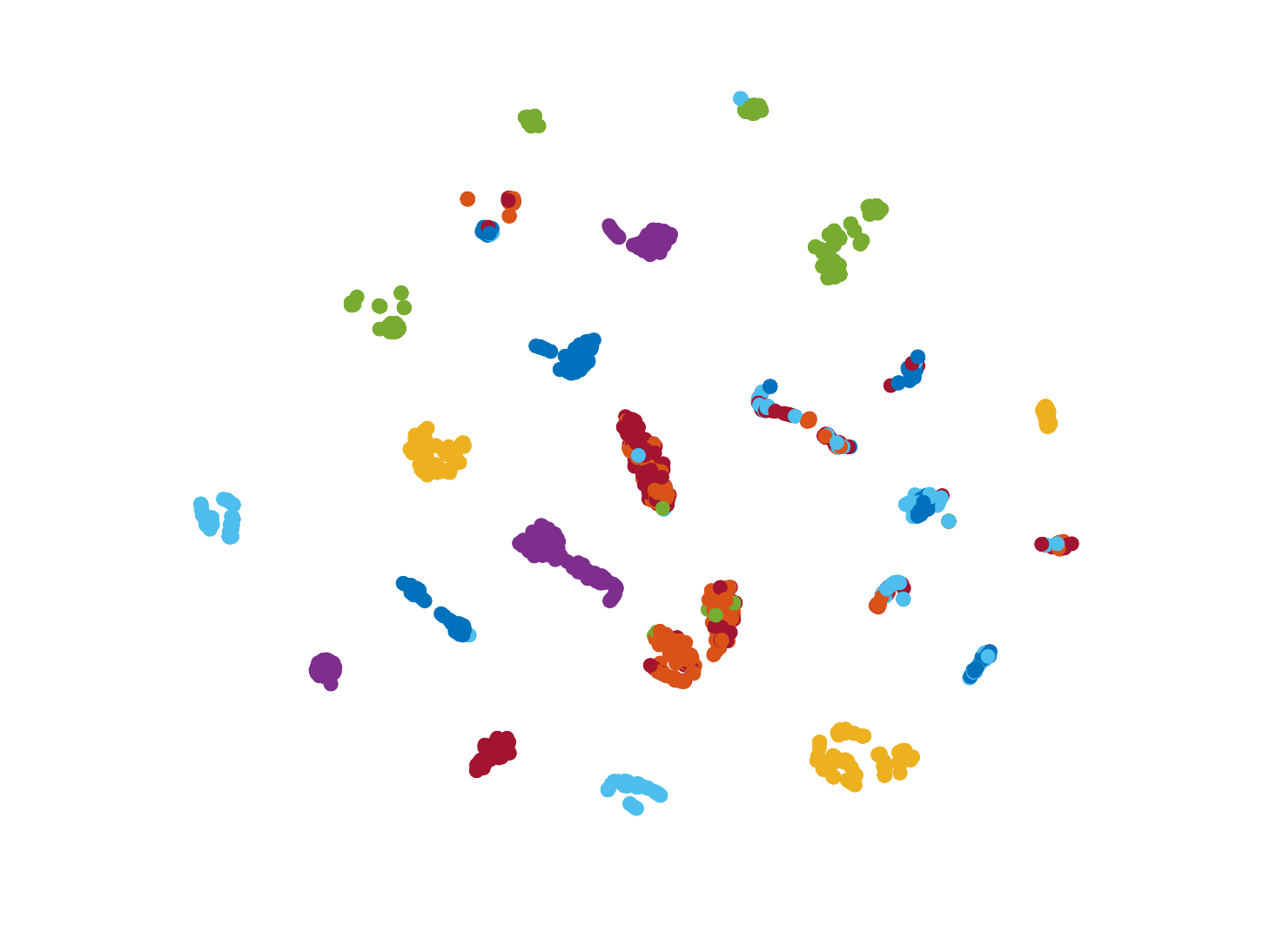}
        }
    }

    \caption{Visualization of SEGMENT by t-SNE. }
    \label{figure_visualization}
    
\end{figure*}

When AnchorGAE is trained, we can perform some fast clustering algorithms 
on $\bm B$ or $\bm Z$, which is elaborated in Section \ref{section_fast_clustering}. 
The optimization of AnchorGAE is summarized in Algorithm \ref{alg_anchor_gae}.

\subsection{Obtain Clustering Assignments} \label{section_fast_clustering}
After training AnchorGAE, we can simply run $k$-means on the learned embedding.
Due to Assumption \ref{assumption_generative}, Euclidean distances of similar points will be small so
that it is appropriate to use $k$-means on the deep representations. 

Another method is to perform spectral clustering on the constructed 
graph, which is the same as anchor-based spectral clustering. 
Let 
\begin{equation}
    \bm{\mathcal{L}} = \bm I - \bm D^{-\frac{1}{2}} \bm A \bm D^{-\frac{1}{2}} = \bm I - \bm A = \bm I - \bm B \bm \Delta^{-1} \bm B^T
\end{equation}
be the normalized Laplacian matrix. 
It should be emphasized that $\bm{\mathcal{L}}$ is also the unnormalized Laplacian matrix.
Accordingly, the objective of spectral clustering is 
\begin{equation}
    \max \limits_{\bm F^T \bm F = \bm I} {\rm tr}(\bm F^T \bm B \bm \Delta^{-1} \bm B^T \bm F), 
\end{equation}
where $\bm F \in \mathbb{R}^{n \times c}$ is the soft indicator matrix. 
Let $\hat{\bm B} = \bm B \bm \Delta^{-\frac{1}{2}}$. 
Then the eigenvalue decomposition of $\bm A$ can be transformed into 
the singular value decomposition of $\hat{{\bm B}}$, 
which reduces the complexity from $\mathcal O(n^2 c)$ to $\mathcal O(mnc)$. 

Finally, one can perform clustering on the bipartite graph directly 
rather than the constructed graph. Let the bipartite graph be 
\begin{equation}
    \bm W = 
    \left [
    \begin{array}{c c}
        \bm 0 & \bm B \\
        \bm B^T & \bm 0 
    \end{array}
    \right ]
    \in \mathbb{R}^{(n+m) \times (n+m)} .
\end{equation}
The degree matrix is represented as 
\begin{equation}
    \bm D = 
    \left [
    \begin{array}{c c}
        \bm D_v & \bm 0 \\
        \bm 0 & \bm D_u 
    \end{array}
    \right ] 
    .
\end{equation}
Then the normalized cut used spectral clustering is to optimize 
\begin{equation}
    \max \limits_{\bm F^T \bm F = \bm I} {\rm tr}(\bm F^T \bm D^{-\frac{1}{2}} \bm W \bm D^{-\frac{1}{2}} \bm F^T) ,
\end{equation}
where 
\begin{equation}
    \bm F = 
    \left [
    \begin{array}{c}
        \bm V \\
        \bm U
    \end{array}
    \right ]
    .
\end{equation}
Since $\bm D_v = \bm I$,
the above problem can be further reformulated as 
\begin{equation}
    \max \limits_{\bm U^T \bm U + \bm V^T \bm V = \bm I} {\rm tr}(\bm V^T \bm B \bm D_u^{-\frac{1}{2}} \bm U).
\end{equation}
The closed-form solution can be calculated by 
\begin{equation}
    \left \{
        \begin{array} {l}
            \bm V = \frac{\sqrt 2}{2} \tilde{\bm V} \\
            \bm U = \frac{\sqrt 2}{2} \tilde{\bm U}, \\
        \end{array}
    \right .
\end{equation}
where $\tilde{\bm V}$ and $\tilde{\bm U}$ are $c$ leading left and right 
singular vectors of $\bm B \bm D_u^{-\frac{1}{2}}$. 
The details can be found in \cite{co-clustering}.
In our experiments, we report the results obtained by the above technique 
based on bipartite graphs, since it usually outperforms the simple $k$-means.

\begin{figure*}[t]
    \centering
    \subcaptionbox{SEGMENT}{
        \label{SEGMENT_mChange}
        \includegraphics[width=0.31\linewidth]{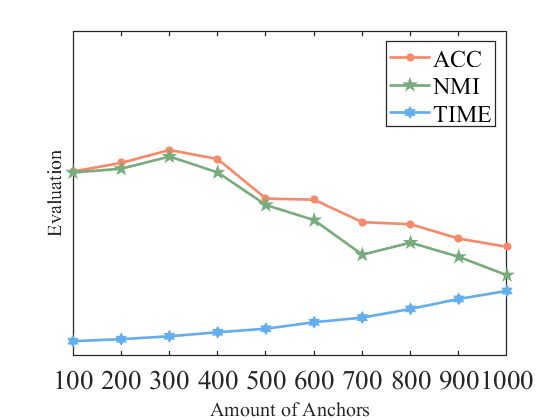}
    }
    \subcaptionbox{USPS}{
        \label{USPS_mChange}
        \includegraphics[width=0.31\linewidth]{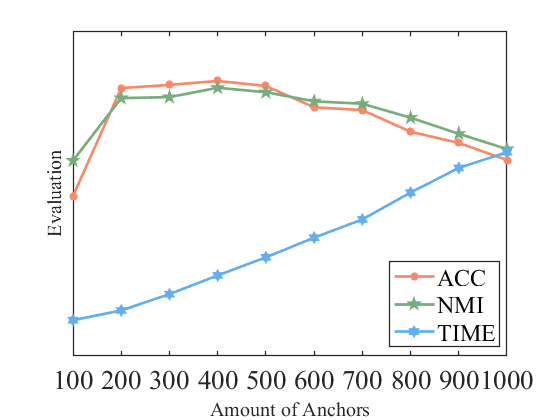}
    }
    \subcaptionbox{MNIST-full}{
        \label{MNIST_TEST_mChange}
        \includegraphics[width=0.31\linewidth]{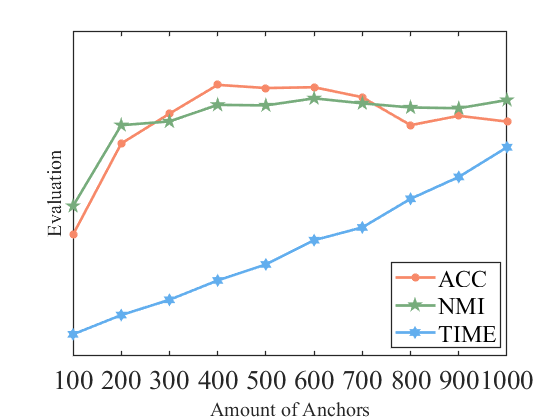}
    }

    \subcaptionbox{ISOLET}{
        \includegraphics[width=0.31\linewidth]{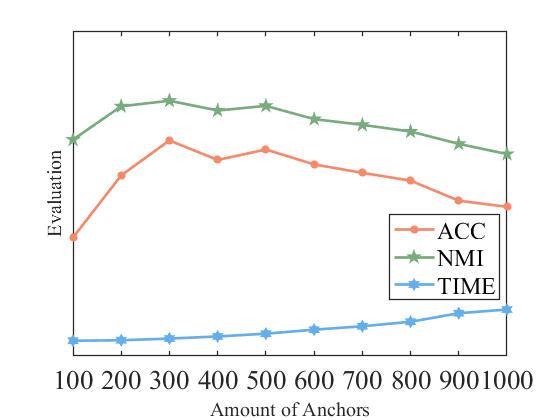}
    }
    \subcaptionbox{MNIST-test}{
        \includegraphics[width=0.31\linewidth]{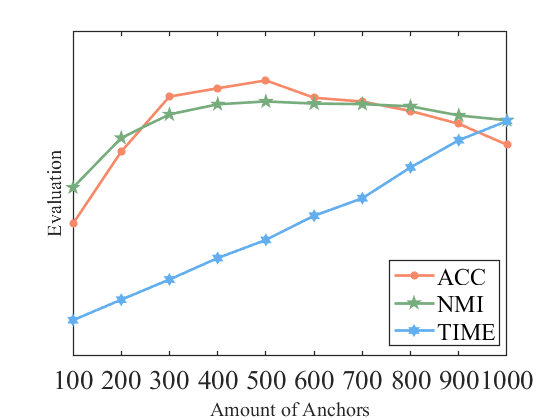}
    }
    \subcaptionbox{Fashion-MNIST}{
        \includegraphics[width=0.31\linewidth]{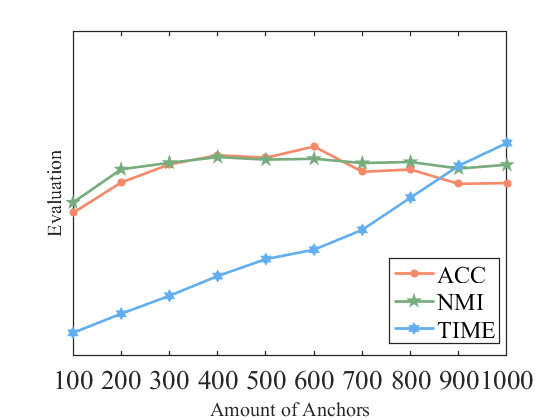}
    }
    \caption{The influence of number of anchors to ACC, 
    NMI, and TIME on 6 datasets.}
    \label{figure_anchor_change}
\end{figure*}

\begin{table}[t]
    \centering
    \setlength{\tabcolsep}{2.5mm}
    \renewcommand\arraystretch{1.2}
    \caption{Information of Datasets}
    \label{table_datasets}
    \begin{tabular}{l c c c}
    \hline
    
    \hline
        Dataset & \# Feature & \# Size & \# Classes \\
    \hline
    \hline
        ISOLET & 617 & 1560 & 26 \\
        SEGMENT & 19 & 2310 & 7 \\
        USPS & 256 & 9298 & 10 \\
        MNIST-test & 784 & 10000 & 10 \\
        MNIST-full & 784 & 70000 & 10 \\
        Fashion-MNIST & 784 & 70000 & 10 \\
    \hline
    \end{tabular}
    
\end{table}

\section{Experiment}
In this section, the experimental results and parametric settings of \textit{AnchorGAE} are illustrated. 
According to these experimental results, we analyze the performance of \textit{AnchorGAE} objectively. 
Additionally, the visualization of data is displayed. 
The experimental results strongly support the theoretical analysis. 

\subsection{Datasets and Comparative Methods}
The performance of the \textit{AnchorGAE} is evaluated on 6 non-graph datasets, 
including 2 UCI datasets (\textit{ISOLET} \cite{UCI}
and \textit{SEGMENT} \cite{UCI})
and 4 image datasets (\textit{USPS} \cite{USPS}, 
\textit{MNIST-test} \cite{MNIST}, 
\textit{MNIST-full} \cite{MNIST}, 
and \textit{Fashion-MNIST} \cite{FashionMNIST}). 
Remark that the two UCI datasets show the difference between AnchorGAE and 
CNN-based clustering models. Since each data point can only be represented 
by a vector, the CNN-based model can not be applied to these datasets. 
Note that Fashion-MNIST is denoted by \textit{Fashion} in Table \ref{table_datasets}
The details of these datasets are shown in Table \ref{table_datasets}.

\textit{AnchorGAE} is compared with 9 methods, including 5 anchor-based 
methods (\textit{Nystrom} \cite{Nystrom}, \textit{CSC} \cite{CSC}, 
\textit{KASP} \cite{KASP}, \textit{LSC} \cite{LSC}, and \textit{SNC} \cite{SNC}), 
3 GAE-based methods (\textit{SGC} \cite{SGC}, 
\textit{GAE} \cite{GAE} and \textit{ClusterGCN} \cite{ClusterGCN}), 
2 deep methods (\textit{DEC} \cite{DEC} and \textit{SpectralNet} \cite{SpectralNet}), 
and classical \textit{K-Means} \cite{KMeans}. 
All codes are downloaded from the homepages of authors.

\subsection{Experimental Setup}
There are three clustering evaluation criteria used in our experiments, 
including the clustering accuracy (\textit{ACC}), normalized mutual information (\textit{NMI}), 
and running \textit{time}. 
In our experiments, the number of anchor points is set from 100 to 1000 and 
the increasing step is 100. 
The best results are reported in Table \ref{table_results}. 
The activation function of the last layer is set to linear, 
while the other layers use the ReLU function. 
Particularly, \textit{ClusterGCN} recommends 4 layers convolution. 
To ensure the convolution structure of \textit{ClusterGCN} the same 
as others, we modify the 4 layers of convolution structure 
(denoted by \textit{ClusterGCN-L4}) to 2 layers 
(denoted by \textit{ClusterGCN-L2}) and 
remain 4 layers of the convolution structure as a competitor. 
To apply ClusterGCN to unsupervised tasks, we use the same GAE loss and 
decoder architecture that is same as the other models. 
To test the scalability of \textit{AnchorGAE} in a deeper network, 
we design 2 layers (denoted by \textit{AnchorGAE-L2}) and 4 layers 
(denoted by \textit{AnchorGAE-L4}) convolution structure separately. 
The 2 layers neurons are designed as 128 and 64 or 256 and 32. 
For the 4 layers neurons, we set the first three layers as 16 and set 
the last layer as $L$+1. 
$L$ is the number of labels. 
The maximum iterations to update network are 200, while
the maximum epochs to update anchors are set as 5. 
The initial sparsity $k_0$ is set as 3. 

\begin{figure}[t]
    \centering
    \includegraphics[width=0.95\linewidth]{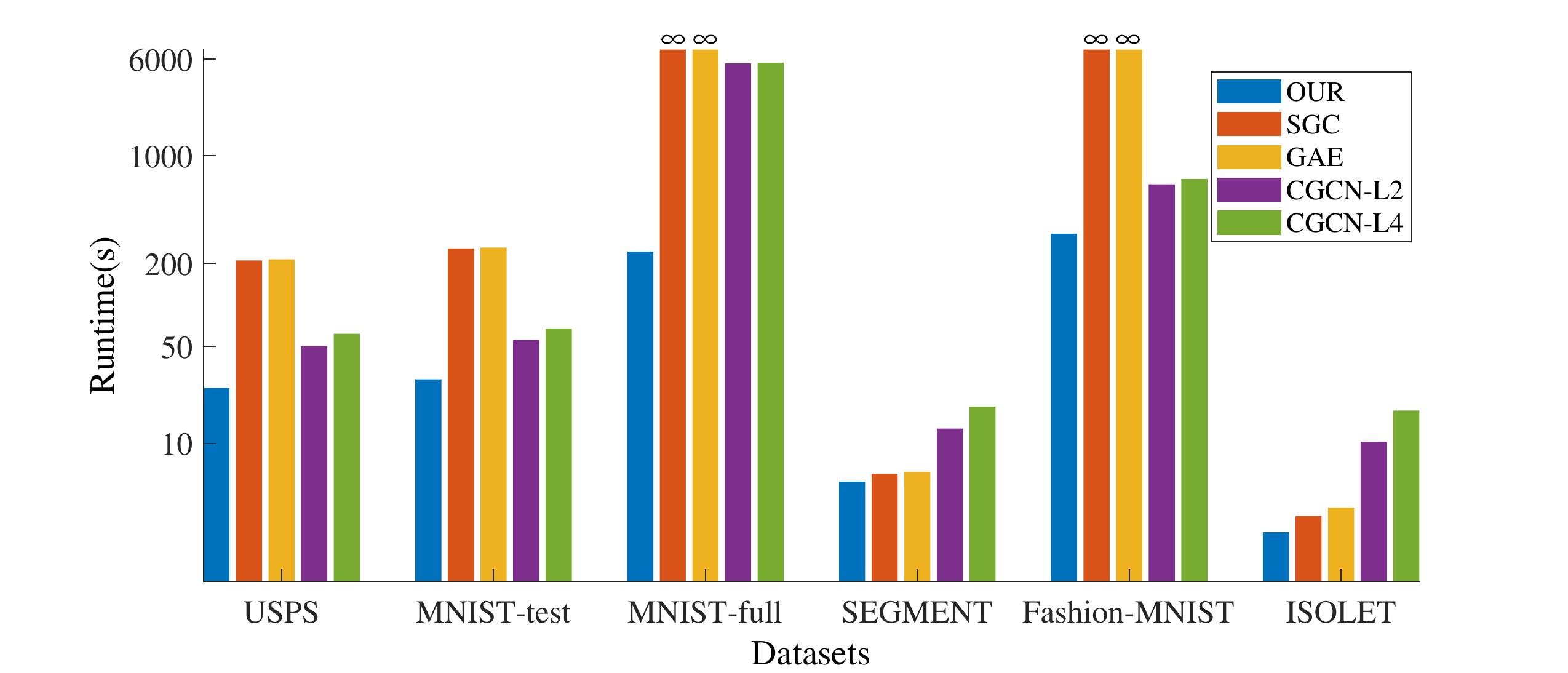}
    \caption{Runtime of several GNN-based methods. 
    If the model can not run on some datasets, it is marked by $\infty$. }
    \label{TIME}
\end{figure}

To investigate the impact of anchors, we run AnchorGAE with different numbers of 
anchors and initial sparsity, \textit{i.e.}, $m$ and $k_0$. 
We also conduct 3 \textbf{ablation experiments} based on AnchorGAE-L2 to study the impact of different parts. 
Firstly, we fix anchors and transition probability during training, 
which is denoted by \textit{Ours-A}.
Secondly, we fix the sparsity $k$, which is denoted by \textit{Ours-B}. 
Thirdly, we use a KNN graph, instead of a weighted graph, with the adaptive update 
and incremental $k$ in our model, which is denoted by \textit{Ours-C}. 
All methods are run 10 times and the average results are reported. 
All experiments are conducted on a PC with an NVIDIA GeForce GTX 1660 GPU SUPER.

\begin{figure}[t]
    \centering
    \subcaptionbox{SEGMENT}{
        \includegraphics[width=0.47\linewidth]{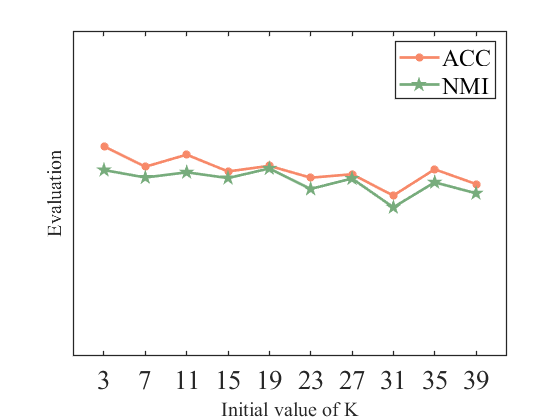}
    }
    \subcaptionbox{MNIST-test}{
        \includegraphics[width=0.47\linewidth]{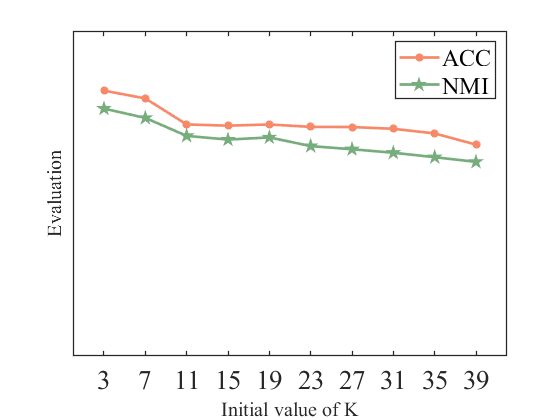}
    }
    \caption{The influence of $k_0$ to
    ACC and NMI on SEGMENT and MNIST-test.}
    \label{figure_k_change}
\end{figure}

\subsection{Experimental Results}

ACC and NMI are shown in Table \ref{table_results} while  
the runtime is reported in Figure \ref{TIME}. 
For ACC and NMI, the best results are bolded and suboptimal results are 
underlined. 
Especially, some methods either spend too much memory to run on 
\textit{MNIST-full} and \textit{Fashion-MNIST}, or the codes provided 
by authors fail to run on them. 
So they are denoted by a horizontal line in Table \ref{table_results}. 
Similarly, they are represented as $\infty$ in Figure \ref{TIME}. 
Besides, the visualization of SEGMENT is shown in Figure \ref{figure_visualization}.
Figure \ref{figure_anchor_change} shows the 
impact of the number of anchors while the effect of the initial sparsity is illustrated in 
Figure \ref{figure_k_change}.

From the experimental results, we can conclude that:
\begin{enumerate}
\item 
On almost all datasets, 
AnchorGAE-L2 obtains the best results in all metrics. 
On \textit{MNIST-full}, \textit{AnchorGAE} obtains 
the suboptimal NMI, 
while DEC, which consists of much more layers, outperforms AnchorGAE. 
Meanwhile, AnchorGAE, with only 2 layers, achieves better ACC than DEC. 
\textit{AnchorGAE} needs to compute anchors 
such that it requires more time. 
\item 
On \textit{MNIST-full}, normal GAE-based methods fail to 
run due to \textit{out-of-memory} since these methods 
need to save the graph structure in each convolution. 
\textit{AnchorGAE} can convolute without constructing 
a graph explicitly, 
which not only needs less memory but also reduces runtime.
\item 
When sparsity $k$ is fixed or anchors and transition 
probability (\textit{i.e.}, matrix $\bm B$) is not updated, 
AnchorGAE usually gets poor results on many datasets. 
For instance, \textit{Ours-B} shrinks about 55\% in ACC 
and 55\% in NMI on MNIST-full, 
while \textit{Ours-A} reduces about 40\% in ACC and 
reduces about 15\% in NMI. 
\item 
To ensure fair competition and distinguish between anchor-based 
methods and GAE-based methods, we set the same number of 
anchors for \textit{AnchorGAE-L4} and  other anchor-based 
methods.
Note that the first two layers of AnchorGAE-L4 are the same as AnchorGAE-L2.
Through observing the results of experiments, 
AnchorGAE-L4 obtains acceptable performance in most 
of the datasets but not the best results. 
It may be caused by the too small dimension 
of the final embedding.
\item 
If the initial sparsity, $k_0$, is set as a too large number, 
then it results in dense connections. 
As a result, a great quantity of wrong information is captured. 
\item 
From the comparison between \textit{AnchorGAE} and \textit{Ours-C}, 
we find that the weighted graphs are still important for clustering 
even after introducing non-linear GNN mapping modules. 
\end{enumerate}

\section{Proof of Theorem \ref{theo}}
For an arbitrary probability distribution $\{p(u_j | v_i)\}_{j=1}^m$, 
let $p_{k}(\cdot | v_i)$ be the $k$-largest one. 
For simplicity, let  
$d_{ij} = d(u_j, v_i) = \|f(\bm x_i) - f(\bm c_j)\|_2^2$ 
where $f(\cdot)$ represents the mapping function implemented by the trained GCN encoder. 
Note that the connectivity distribution is estimated through the current representations, 
\textit{i.e.}, 
\begin{equation}
    p^{(0)}(\cdot | v_i) = \arg \min _{p(\cdot | v_i)} \mathbb{E}_{u_j \sim p(\cdot | v_i)} \hat d_{ij} + \gamma_i \sum _{j=1}^m (p(u_j | v_i) - \frac{1}{m})^2.
\end{equation}
where $\hat d_{ij} = \|\hat f(\bm x_i) - f(\bm c_j)\|_2^2$ and 
$\hat f(\cdot)$ denotes the mapping function before the new training of GCN encoders.

We define $p^{(t)}(u_j | v_i)$ and $f^{(t)}(\bm c_j)$ 
as the solutions at the $t$-th iteration. 
To prove the theorem clearly, the iterative update of $p(\cdot | v_i)$ and anchors 
is shown in Algorithm \ref{alg_anchor}. 
The following lemma aims to depict the first estimated connectivity distribution
after one \textit{perfect} training of GCNs. 

\begin{myLemma} \label{lemma_uniform}
    Suppose that 
    $$p^{(1)}(\cdot | v_i) = \arg \min _{p(\cdot | v_i)} \mathbb{E}_{u_j \sim p(\cdot | v_i)} d_{ij} + \gamma_i \sum _{j=1}^m (p(u_j | v_i) - \frac{1}{m})^2$$
    where $\gamma_i$ ensures the sparsity as $k$.
    Let $p_k^{(0)} (\cdot | v_i) = \mu$. If
    $|p_j^{(0)}(\cdot | v_i) - q_j(\cdot | v_i)| \leq \varepsilon \leq \mu$ for any $i$, 
    then $\forall j \leq k$, 
    \begin{equation}
        | p^{(1)}_j(\cdot | v_i) - \frac{1}{k}| \leq \mathcal{O}(\frac{\log (1 / \mu)}{\log (1 / \varepsilon)}) .
    \end{equation}
\end{myLemma}
\begin{proof}
    Without loss of generality, we focus on the connectivity distribution 
    of $v$ and suppose that 
    $p^{(0)}(u_1 | v) \geq p^{(0)}(u_2 | v) \geq \cdots \geq p^{(0)}(u_k | v) > 0 = p^{(0)}(u_{k+1} | v) = \cdots = p^{(0)}(u_m | v)$. 
    Let $p_i = p^{(0)}(u_i | v)$ and $q_i = q(u_i | v)$. According to the definitions, 
    \begin{equation}
        \begin{aligned}
            q_i = \frac{\exp(- d_i)}{\sum_{j=1}^m \exp(- d_j)} .
        \end{aligned}
    \end{equation}
    If for any $i$, we have $|p_i - q_i| \leq \varepsilon$. Clearly, $p_{k+1} = 0$. Suppose that $q_{k+1} = \tau \leq \varepsilon$. Therefore, we have
    \begin{equation}
        \begin{split}
            & \frac{\exp(- d_{k+1})}{\sum_{j=1}^m \exp(- d_j)} = \tau 
            \Leftrightarrow ~  d_{k+1} = \log \frac{1}{\tau} - \log C .
        \end{split}
    \end{equation}
    where $C = \sum_{j=1}^m \exp(- d_j)$. 
    Combine with the condition, $p_k \geq \mu$, and we have
    \begin{equation}
        \begin{split}
            & \frac{\exp(- d_{k})}{\sum_{j=1}^m \exp(- d_j)} \geq p_k - \varepsilon \geq \mu - \varepsilon \\
            \Rightarrow ~ &  d_k \leq -\log C - \log (\mu - \varepsilon) . \\
        \end{split}
    \end{equation}
    Similarly, since $p_1 \geq \frac{1}{k}$, 
    \begin{equation}
        \begin{split}
            & \frac{\exp(- d_{1})}{\sum_{j=1}^m \exp(- d_j)} \leq 1 
            \Rightarrow ~  d_1 \geq - \log C .\\
        \end{split}
    \end{equation}
    If we update the connectivity distribution based on $\{ d_{i}\}_{i = 1}^m$, then for any $i \leq k$,
    \begin{equation}
        p_i^{(1)} = \frac{ d_{k+1} -  d_i}{\sum _{j=1}^k ( d_{k+1} -  d_j)} .
    \end{equation}
    Furthermore, for any $i, j \leq k$,
    \begin{align*}
        | p_i^{(1)} -  p_j^{(1)}| & = \frac{| d_j -  d_i|}{\sum _{j=1}^k ( d_{k+1} -  d_j)} \\
        & \leq \frac{| d_k -  d_1|}{\sum _{j=1}^k ( d_{k+1} -  d_j)} \\
        & = \frac{-\log (\mu - \varepsilon)}{\sum _{j=1}^k (\log \frac{1}{\tau} - \log C -  d_j)} \\
        & \leq \frac{-\log (\mu - \varepsilon)}{\sum _{j=1}^k (\log \frac{1}{\tau} - \log C -  d_k)} \\
        &\leq \frac{-\log (\mu - \varepsilon)}{k (\log(\mu - \varepsilon) - \log \tau)} \\
        & = \frac{1}{k} \cdot \frac{\log (\mu - \varepsilon)}{ \log \tau - \log(\mu - \varepsilon)} \\
        & = \frac{1}{k} \cdot \frac{1}{\frac{\log (1/\tau)}{\log (1/(\mu - \varepsilon))} - 1} \\
        & \leq \frac{1}{k} \cdot \frac{1}{\frac{\log (1 / \varepsilon)}{\log 1 / (\mu - \varepsilon)} - 1} 
        = \mathcal{O}(\frac{\log (1 / \mu)}{\log (1 / \varepsilon)}) . 
    \end{align*}
    For any $i \leq k$, we have the following formulation based on the above result 
    \begin{align*}
        |p_i^{(1)} - \frac{1}{k}| & \leq |p_i^{(1)} - \frac{1}{k} \sum _{j=1}^k p_j^{(1)}| \\
        & \leq \frac{1}{k} |\sum _{j=1}^k (p_i^{(1)} - p_j^{(1)})| \\
        & \leq \frac{1}{k} \sum _{j=1}^k | (p_i^{(1)} - p_j^{(1)}) | \\
        & \leq \mathcal{O}(\frac{\log (1 / \mu)}{\log (1 / \varepsilon)}) . 
    \end{align*}
    The proof is easy to extend to other nodes. Hence, the theorem is proved.
\end{proof}

The following lemma shows that the updated anchors will approximate the mean vectors 
of the connected nodes if $p^{(t)}(u_j | v_i)$ approaches the sparse and uniform distribution.

\begin{myLemma} \label{lemma_mean}
    Let $\mathcal{N}_j$ denote the nodes that connect the anchor $u_j$ and 
    $p(u_j | v_i)$ represent a $k$-sparse solution given by Theorem \ref{theo}. 
    If $|p_j(\cdot | v_i) - \frac{1}{k}| \leq \varepsilon$ for any $j \leq k$, then 
    \begin{equation}
        \|\frac{\sum _{i=1}^n p(u_j | v_i) f(\bm x_i)}{\sum _{i=1}^n p(u_j | v_i)} - \frac{1}{n_j} \sum \limits_{i \in \mathcal{N}_j} f(\bm x_i) \|_2^2 \leq \mathcal O(\varepsilon) .
    \end{equation}
\end{myLemma}
\begin{proof}
    To keep the notation uncluttered, let $\bm f_i = f(\bm x_i)$ and 
    \begin{equation}
        \bm h_j = \frac{\sum _{i=1}^n p(u_j | v_i) f(\bm x_i)}{\sum _{i=1}^n p(u_j | v_i)} .
    \end{equation}
    Note that $\sum _{i=1}^n p(u_j | v_i) f(\bm x_i) = \sum _{i \in \mathcal{N}_j} p(u_j | v_i) f(\bm x_i)$.
    Let $n_j = |\mathcal{N}_j|$ and we have 
    \begin{equation}
        \|\bm h_j - \frac{1}{n_j} \sum \limits_{i \in \mathcal{N}_j } \bm f_i \|_2
        = \frac{1}{s_j} \|\sum \limits_{i \in \mathcal{N}_j} (p(u_j | v_i) - \frac{s_j}{n_j}) \bm f_i \|_2, 
    \end{equation}
    where $s_j = \sum _{i=1}^n p(u_j | v_i)$. According to the condition, 
    we have 
    \begin{align*}
        & \frac{1}{k} - \varepsilon \leq p(u_j | v_i) \leq \frac{1}{k} + \varepsilon \\
        ~ \Rightarrow ~ & \frac{n_j}{k} - n_j \cdot \varepsilon \leq s_j \leq \frac{n_j}{k} + n_j \cdot \varepsilon .   
    \end{align*}
    Accordingly, we have 
    \begin{align*}
            & \|\bm h_j - \frac{1}{n_j} \sum \limits_{i \in \mathcal{N}_j } \bm f_i \|_2 \\
            \leq & \sum \limits_{i \in \mathcal{N}_j} \frac{1}{s_j} |p(u_j | v_i) - \frac{s_j}{n_j}| \cdot \|\bm f_i\|_2 \\
            \leq & \sum \limits_{i \in \mathcal{N}_j} \frac{\|\bm f_i\|_2}{s_j} \max \{|p(u_j | v_i) - \frac{1}{k} + \varepsilon|, |p(u_j | v_i) - \frac{1}{k} - \varepsilon|\} \\
            \leq & \sum \limits_{i \in \mathcal{N}_j} \frac{\|\bm f_i\|_2}{s_j} (|p(u_j | v_i) - \frac{1}{k}| + |\varepsilon|) \\
            \leq & \sum \limits_{i \in \mathcal{N}_j} \frac{2 \|\bm f_i\|_2}{s_j} \varepsilon = \mathcal{O}(\varepsilon).
    \end{align*}
    Hence, the theorem is proved. 
\end{proof}
With the following lemma, we can prove Theorem \ref{theo} via the mathematical induction method.

\begin{myLemma} \label{lemma_disjoint}
    Suppose that $q_j(\cdot | v_i) \leq \varepsilon$ for $j > k$,
    and $d(v_i, u_j) \leq M$. 
    Then there exists $\delta > 0$ such that
    if $\varepsilon \leq \delta$ and $v_i$ does not connect with $u_j$, then for any $v' \in \mathcal{N}_j$, 
    $v_i$ and $v'$ share no anchors. 
\end{myLemma}
\begin{proof}
    At first, let $\mu = \min_i q_k(\cdot | v_i)$.
    Hence, $q_k(\cdot | v_i) \geq \mu \geq \varepsilon$.
    For an arbitrary node $v_i$ connected with $u_a$ and $u_b$ 
    and an anchor $u_c$ which disconnects with $v_i$, 
    suppose that there exists a node $v_j$ such that $v_j$ 
    connects with $u_b$ and $u_c$ but disconnects with $u_a$.
    According to the triangular inequality, we have 
    \begin{align*}
        \frac{\exp(-d(u_a, u_c))}{\exp(-d(v_i, u_a))} 
        & \leq \frac{\exp(-d(v_i, u_c) + d(v_i, u_a))}{\exp(-d(v_i, u_a))} \\
        & =\frac{\exp(-d(v_i, u_c))}{\exp(-2d(v_i, u_a))} \\
        & = \frac{q(u_c | v_i)}{q(u_a | v_i) \exp(-d(v_i, u_a))} \\
        & \leq \frac{\varepsilon}{\mu} e^M .
    \end{align*}
    According to the triangular inequality, 
    \begin{align*}
        & d(v_i, u_a) + \log \frac{\mu}{\varepsilon} - M \\
        \leq ~ & d(u_a, u_c) < d(u_a, u_b) + d(u_b, u_c) \\
        < ~ & d(v_i, u_a) + d(v_i, u_b) + d(v_j, u_c) + d(v_j, u_b) . 
    \end{align*}
    Therefore, 
    \begin{eqnarray} \label{ineq_contradiction}
        \log \frac{\mu}{\varepsilon} < 4 M . 
    \end{eqnarray}
    If we define $\delta$ as 
    \begin{equation}
        \delta = \min \{\mu \cdot \exp(-4 M), \mu\}, 
    \end{equation}
    then Ineq. (\ref{ineq_contradiction}) results in 
    a contradiction when $\varepsilon \leq \delta$.
\end{proof}

\begin{proof}[Proof of Theorem \ref{theo}]
    The proof is to analyze each iteration of step 6.
    Let $\mu = \min_i p_k^{(1)}(\cdot | v_i)$ and we can obtain that 
    $| p_j^{(1)}(\cdot | v_i) - \frac{1}{k}| \leq \mathcal{O}(\frac{\log (1 / \mu)}{\log (1 / \varepsilon)})$ 
    provided that $\varepsilon \leq \mu$, according to Lemma \ref{lemma_uniform}. 
    Certainly, the conclusion constructs the precondition of Lemma \ref{lemma_mean}.

    Let $\bm \theta_j = \frac{1}{|\mathcal{N}_j|} \sum_{\bm x_i \in \mathcal{N}_j} f(\bm x_i)$.
    According to Lemma \ref{lemma_mean}, we have 
    $\|f^{(1)}(\bm c_j) - \bm \theta_j\|_2^2 \leq \mathcal{O}(\frac{\log (1 / \mu)}{\log (1 / \varepsilon)})$ 
    and we can further derive that 
    \begin{align*}
        \|f(\bm x_i) - f^{(1)}(\bm c_j)\|_2 
        &\leq \|f(\bm x_i) - \bm \theta_j\|_2 
         + \|f^{(1)}(\bm c_j) - \bm \theta_j\|_2 \\
        & \leq \|f(\bm x_i) - \bm \theta_j\|_2 + \mathcal{O}(\sqrt{\frac{\log (1 / \mu)}{\log (1 / \varepsilon)}}) , 
    \end{align*}
    and
    \begin{align*}
            \|f(\bm x_i) - f^{(1)}(\bm c_j)\|_2 
            & \geq \|f(\bm x_i) - \bm \theta_j\|_2 - \|f^{(1)}(\bm c_j) - \bm \theta_j\|_2 \\
            & \geq \|f(\bm x_i) - \bm \theta_j\|_2 - \mathcal{O}(\sqrt{\frac{\log (1 / \mu)}{\log (1 / \varepsilon)}})  .
    \end{align*}
    Therefore, for $d^{(1)}(v_i, u_j)$, we have 
    \begin{align*}
        & \mathcal{O}(\frac{\log (1 / \mu)}{\log (1 / \varepsilon)}) - 2 \|f(\bm x_i) - \bm \theta_j\|_2 \mathcal{O}(\sqrt{\frac{\log (1 / \mu)}{\log (1 / \varepsilon)}}) \\
        \leq & d^{(1)}(v_i, u_j) - \|f(\bm x_i) - \bm \theta_j\|_2^2 \\
        \leq & \mathcal{O}(\frac{\log (1 / \mu)}{\log (1 / \varepsilon)}) + 2 \|f(\bm x_i) - \bm \theta_j\|_2 \mathcal{O}(\sqrt{\frac{\log (1 / \mu)}{\log (1 / \varepsilon)}}) .
    \end{align*}
    Due to that $d^{(1)}(v_i, u_j)$ is upper-bounded, 
    the above conclusions are provided as the preconditions of Lemma \ref{lemma_disjoint}.
    To keep simplicity, let $r = \|f(\bm x_i) - \bm \theta_j\|_2$. 
    Therefore, we know that for any $l \leq k$, 
    \begin{align*}
        & d^{(1)}_{k+1}(v_i, \cdot) - d^{(1)}_l (v_i, \cdot) \\
        \leq ~ & d^{(1)}_{k+1}(v_i, \cdot) - r^2 
        - \mathcal{O}(\frac{\log (1 / \mu)}{\log (1 / \varepsilon)}) + 2 r \mathcal{O}(\sqrt{\frac{\log (1 / \mu)}{\log (1 / \varepsilon)}}) , 
    \end{align*}
    and
    \begin{align*}
            & d^{(1)}_{k+1}(v_i, \cdot) - d^{(1)}_l (v_i, \cdot) \\ 
            \geq ~ & d^{(1)}_{k+1}(v_i, \cdot) - r^2 
            - \mathcal{O}(\frac{\log (1 / \mu)}{\log (1 / \varepsilon)}) - 2 r \mathcal{O}(\sqrt{\frac{\log (1 / \mu)}{\log (1 / \varepsilon)}}) .  
    \end{align*}
    Furthermore, as 
    \begin{align*}
            & p^{(2)}(u_j | v_i) \\
            = & \frac{(d^{(1)}_{k+1}(v_i, \cdot) - d^{(1)}(v_i, \cdot) )_+}{\sum _{l=1}^k (d^{(1)}_{k+1}(v_i, \cdot) - d^{(1)}_l (v_i, \cdot))} \\
            \leq & \frac{d^{(1)}_{k+1}(v_i, \cdot) - r^2 - \mathcal{O}(\frac{\log (1 / \mu)}{\log (1 / \varepsilon)}) + 2 r \mathcal{O}(\sqrt{\frac{\log (1 / \mu)}{\log (1 / \varepsilon)}})}{\sum _{i=1}^k [d_{k+1}^{(1)}(v_i, \cdot) - r^2] - k \mathcal{O}(\frac{\log (1 / \mu)}{\log (1 / \varepsilon)}) - 2 k r \mathcal{O}(\sqrt{\frac{\log (1 / \mu)}{\log (1 / \varepsilon)}})},
    \end{align*}
    we have 
    \begin{align*}
            & p^{(2)}(u_j | v_i) - \frac{1}{k} \\ 
            \leq ~ & \frac{4 r \mathcal{O}(\sqrt{\frac{\log (1 / \mu)}{\log (1 / \varepsilon)}})}{k [d_{k+1}^{(1)}(v_i, \cdot) - r^2 - \mathcal{O}(\frac{\log (1 / \mu)}{\log (1 / \varepsilon)}) - 2 r \mathcal{O}(\sqrt{\frac{\log (1 / \mu)}{\log (1 / \varepsilon)}}) ]} \\
            = ~ & \mathcal{O}(\sqrt{\frac{\log (1 / \mu)}{\log (1 / \varepsilon)}}) .
    \end{align*}
    Similarly, we can also infer that 
    \begin{align*}
        & p^{(2)}(u_j | v_i) - \frac{1}{k} \\ 
        \leq ~ & \frac{- 4 r \mathcal{O}(\sqrt{\frac{\log (1 / \mu)}{\log (1 / \varepsilon)}})}{k [d_{k+1}^{(1)}(v_i, \cdot) - r^2 - \mathcal{O}(\frac{\log (1 / \mu)}{\log (1 / \varepsilon)}) + 2 r \mathcal{O}(\sqrt{\frac{\log (1 / \mu)}{\log (1 / \varepsilon)}}) ]} \\
        = ~ & \mathcal{O}(\sqrt{\frac{\log (1 / \mu)}{\log (1 / \varepsilon)}}) .
    \end{align*}
    According to Lemma \ref{lemma_disjoint}, there exists a constant $\delta$ such that when $\varepsilon \leq \delta$, 
    $p^{(2)}(u_j | v_i)$ approximates the sparse and uniform distribution, 
    \begin{equation}
        |p^{(2)} (u_j | v_i) - \frac{1}{k}| \leq \mathcal{O}(\sqrt{\frac{\log (1 / \mu)}{\log (1 / \varepsilon)}}) = \mathcal{O}(\log^{-1/2}(1/\varepsilon)).
    \end{equation}
    We can repeat the above proof for every iteration to update 
    $p^{(t)}(u_j | v_i)$ and $f^{(t)}(\bm c_j)$, and therefore, the theorem 
    is proved.
\end{proof}

\section{Conclusion}
In this paper, we propose an anchor-siamese GCN clustering with 
adaptive $\mathcal O(n)$ bipartite convolution (\textit{AnchorGAE}). 
The generative perspective for weighted graphs helps us to build the 
relationship between the non-graph-type data and the graph-type data, 
and provides a reconstruction goal. 
Since the anchor-based bipartite graph factorizes the adjacency matrix, 
we can rearrange the order of matrix multiplications to accelerate the 
graph convolution. 
Experiments show that AnchorGAE consumes less time especially on 
MNIST-full and obtains impressive performance. 
Ablation experiments also verify the necessity of updating anchors, 
transition probability, and sparsity in a self-supervised way. 
From both theoretical and empirical aspects, we show that the simple update 
would cause a collapse when the model is well-trained. 
The designed mechanism to connect different groups in advance is also 
verified by experiments. 
Since the original images of anchors are estimated roughly, we will focus 
on how to exactly compute the original images from the embedding space 
in future work. 
To further speed up the model, it is worthy 
developing stochastic gradient descent for AnchorGAE in the future.


\bibliographystyle{IEEEtran}
\bibliography{citations.bib}

\begin{IEEEbiography}[{\includegraphics[width=1in,height=1.25in,clip,keepaspectratio]{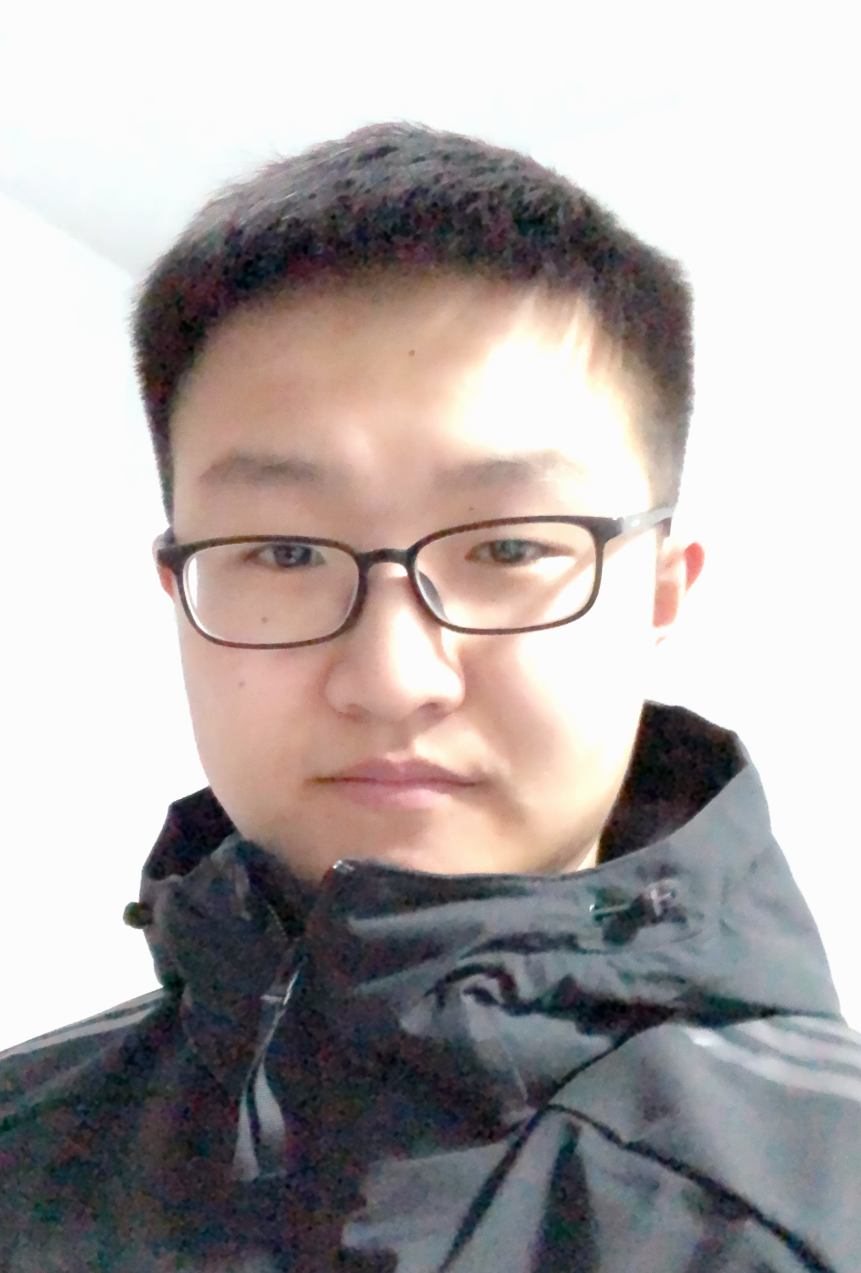}}]{Hongyuan Zhang}
    received the B.E. degree in software engineering from Xidian University, Xi'an, China in 2019. 
    He is currently pursuing the Ph.D. degree from the School of Computer Science and the School of Artificial Intelligence, OPtics and ElectroNics (iOPEN), Northwestern Polytechnical University, Xi'an, China. 
\end{IEEEbiography}

\begin{IEEEbiography}[{\includegraphics[width=1in,height=1.25in,clip,keepaspectratio]{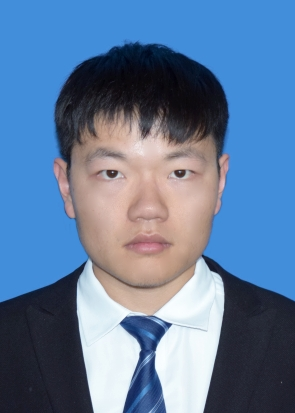}}]{Jiankun Shi}
    received the B.E. degree in software engineering from Henan University, Zhengzhou, China in 2020. 
    He is currently working toward the master's degree with the School of Artificial Intelligence, OPtics and ElectroNics(iOPEN), Northwestern Polytechnical University, Xi'an, China.
\end{IEEEbiography}
    
\begin{IEEEbiography}[{\includegraphics[width=1in,height=1.25in,clip,keepaspectratio]{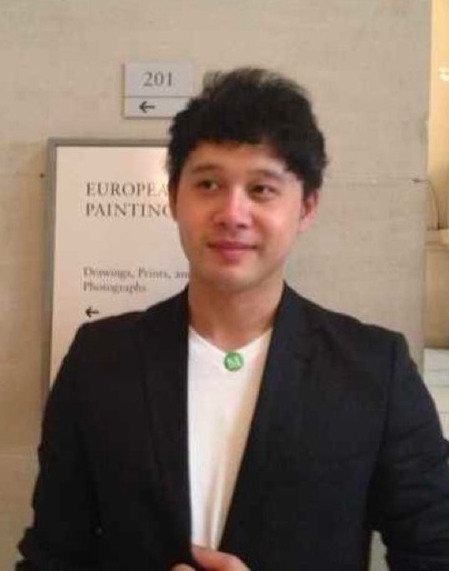}}]{Rui Zhang} (M'19)
    received the Ph.D degree in computer science at Northwestern Polytechnical University, Xi'an, China in 2018. 
    He currently serves as an Associate Professor with the School of Artificial Intelligence, OPtics and ElectroNics (iOPEN), Northwestern Polytechnical University, Xi'an, China.
\end{IEEEbiography}

\begin{IEEEbiographynophoto}{Xuelong Li} (M'02-SM'07-F'12) 
    is a Full Professor with the School of Artificial Intelligence, OPtics and ElectroNics (iOPEN), Northwestern Polytechnical University, Xi'an, China. 
\end{IEEEbiographynophoto}

\end{document}